\pdfoutput=1

\documentclass[11pt]{article}

\usepackage[final]{acl}

\usepackage{times}
\usepackage{latexsym}

\usepackage[T1]{fontenc}

\usepackage[utf8]{inputenc}

\usepackage{microtype}

\usepackage{inconsolata}

\usepackage{graphicx}

\usepackage{amsthm}
\usepackage{amsmath,amsfonts,bm}
\usepackage{graphicx}
\usepackage{mdwlist}
\usepackage{multirow}
\usepackage{booktabs}
\usepackage{xspace}
\usepackage{arydshln}
\usepackage{hhline}
\usepackage{hyperref}

\usepackage{enumitem}
\usepackage{wrapfig}
\usepackage{capt-of}
\usepackage{algorithm}
\usepackage{algorithmic}

\newcommand{\method}{UniQuan\xspace}
\newcommand{\methodf}{UniQuan$_F$\xspace}
\newcommand{\methodfz}{UniQuan$_F$-0\xspace}

\newcommand{\methodfullbold}{\textbf{\underline{Uni}}fied \textbf{\underline{Quan}}tization\xspace}
\newcommand{\methodffullbold}{\textbf{\underline{Uni}}fied \textbf{\underline{Quan}}tization with \textbf{\underline{F}}lexible Mapping\xspace}

\newcommand{\naivef}{\methodfz}
\newcommand{\alternating}{\textsc{Alternating}\xspace}
\newcommand{\alternatings}{\textsc{Alternating}$^*$\xspace}

\newcommand{\blue}[1]{{\color{blue} #1}}

\newtheorem{theorem}{Theorem}

\def\trans{{\mathcal{T}}}
\def\utrans{{\mathcal{T}}_U}
\def\ftrans{{\mathcal{T}_F}}
\def\otrans{{\mathcal{T}_O}}
\def\btrans{{\mathcal{T}_B}}
\def\detrans{{\mathcal{D}}}
\def\fdetrans{{\mathcal{D}}_F}
\def\odetrans{{\mathcal{D}_O}}
\def\bdetrans{{\mathcal{D}_B}}
\def\udetrans{{\mathcal{D}_U}}
\def\map{{\mathcal{M}}}
\def\umap{{\mathcal{M}_U}}
\def\gmap{{\mathcal{M}_B}}
\def\lpgmap{{\mathcal{M}^*_B}}

\def\recon{{\mathcal{R}_B}}

\setlist[itemize]{topsep=-1mm, itemsep=0mm}

\def\vc{{\bm{c}}}
\def\vd{{\bm{d}}}

\def\vq{{\bm{q}}}

\def\vs{{\bm{s}}}

\def\vw{{\bm{w}}}

\def\valpha{{\bm{\alpha}}}

\def\mC{{\bm{C}}}

\def\mX{{\bm{X}}}

\DeclareMathAlphabet{\mathsfit}{\encodingdefault}{\sfdefault}{m}{sl}
\SetMathAlphabet{\mathsfit}{bold}{\encodingdefault}{\sfdefault}{bx}{n}

\def\sD{{\mathbb{D}}}

\def\sR{{\mathbb{R}}}

\DeclareMathOperator*{\argmin}{arg\,min}

\let\llncssubparagraph\subparagraph
\let\subparagraph\paragraph
\usepackage{titlesec}
\let\subparagraph\llncssubparagraph

\newcommand{\smallsection}[1]{\vspace{1mm}\noindent\smash{\textbf{#1.}}}

\title{Unifying Uniform and Binary-coding Quantization
\\ for Accurate Compression of Large Language Models}

\author{
  Seungcheol Park$^1$,
  Jeongin Bae$^2$,
  Beomseok Kwon$^2$,
  Minjun Kim$^1$, \\
  \bf{Byeongwook Kim$^2$,
  Se Jung Kwon$^2$,
  U Kang$^1$\footnote[1]{Corresponding author},
  Dongsoo Lee$^2$} \\
  Seoul National University$^1$ $\quad$
  NAVER Cloud$^2$ \\
  \{ant6si, minjun.kim, ukang\}@snu.ac.kr
}

\begin{document}
\maketitle
\begin{abstract}
How can we quantize large language models while preserving accuracy?
Quantization is essential for deploying large language models (LLMs) efficiently.
Binary-coding quantization (BCQ) and uniform quantization (UQ) are promising quantization schemes that have strong expressiveness and optimizability, respectively.
However, neither scheme leverages both advantages.
In this paper, we propose \textbf{\methodf} (\methodffullbold), an accurate quantization method for LLMs.
\methodf harnesses both strong expressiveness and optimizability by unifying the flexible mapping technique in UQ and BCQ's non-uniform quantization levels.
We propose unified initialization, and local and periodic mapping techniques to optimize the parameters in \methodf precisely.
After optimization, our unification theorem removes computational and memory overhead, allowing us to utilize the superior accuracy of \methodf without extra deployment costs induced by the unification.
Experimental results demonstrate that \methodf outperforms existing UQ and BCQ methods, achieving up to 4.60\% higher accuracy on GSM8K benchmark. 
\end{abstract}
\renewcommand{\thefootnote}{\fnsymbol{footnote}}
\footnotetext[1]{Corresponding author}
\renewcommand{\thefootnote}{\arabic{footnote}}  

\section{Introduction}
\label{sec:intro}
How can we compress large language models without compromising accuracy?
Reducing the size of large language models (LLMs)~\citep{gpt3,llama2,llama3} is crucial for deploying them in real-world applications since they require expensive computational and memory costs.
Quantization algorithms~\citep{xu, alphatuning, llmint8,smoothquant,awq} efficiently compress LLMs via bit-width reduction of weights by encoding them with a small set of values, namely, quantization levels.
Each quantization scheme has its own quantization parameters that determine the quantization levels; accurately optimizing these parameters is important because it guarantees that the quantization levels are well aligned with the model's weight distribution.

Uniform quantization (UQ)~\cite{flexround,awq,omni} and binary-coding quantization (BCQ)~\cite{xu,alphatuning} are promising quantization schemes that ensure fast inference of quantized LLMs.
While UQ divides its quantization levels uniformly, BCQ results in non-uniform quantization levels through the addition and subtraction of values assigned per each bit.
Research on quantizing LLMs with UQ is actively ongoing, but no study has yet explored BCQ on LLMs as fast acceleration kernels~\cite{lut_gemm,shiftadd} that support BCQ scheme have been released recently.
%

\begin{figure*}[t]
    \centering
    \includegraphics[width=1.0\textwidth]{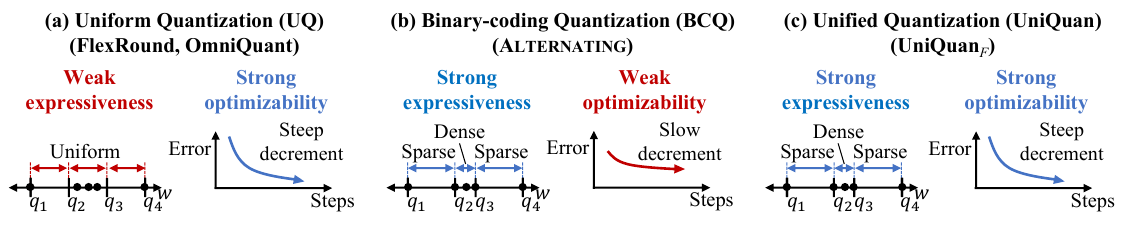}
    \caption{
    A comparison of (a) uniform, (b) binary-coding, and (c) unified quantization schemes.
    UQ has strong optimizability swiftly reducing errors during optimization, and BCQ has strong expressiveness adapting its non-linear quantization levels ($q$, bar) according to the distribution of weights ($w$, circle).
    \method combines the advantages of both schemes by unifying UQ and BCQ.
    }
    \label{fig:crown}
    \vspace{-2mm}
\end{figure*}


We compare the expressiveness and optimizability of quantization schemes for LLMs in Figure~\ref{fig:crown}.
UQ demonstrates strong optimizability, efficiently reducing quantization errors, by leveraging various advanced methods such as FlexRound~\cite{flexround} and OmniQuant~\cite{omni}.
However, its expressiveness is constrained, failing to adapt its quantization levels to the weight distribution due to their uniform spacing.
Conversely, BCQ offers strong expressiveness through non-uniform quantization levels, but its optimizability is limited due to the absence of precise quantization techniques. 
\alternating~\cite{xu} is the only BCQ method applicable to LLMs,
but its major drawback is that it does not consider the input distribution.
In summary, neither scheme fully achieves both strong expressiveness and optimizability.

In this paper, we propose \textbf{\methodf} (\methodffullbold), an accurate quantization method for LLMs.
We find that the strong optimizability of UQ originates from its transformation process, and the strong expressiveness of BCQ originates from its generalized mapping process that maps weights to non-uniform quantization levels (see Section~\ref{subsec:uni}).
Building on this observation, we define \textbf{\method} (\methodfullbold) which unifies UQ and BCQ schemes by integrating UQ's transformation process into BCQ's mapping process, thereby harnessing both strong optimizability and expressiveness as shown in Figure~\ref{fig:crown}(c).
In \methodf, we unify FlexRound and \alternating, the best-performing UQ and BCQ methods, respectively, following the quantization process of \method.
We improve the accuracy of \methodf with two main ideas: 1) unified initialization for joint initialization of quantization parameters from FlexRound and \alternating, and 2) local and periodic mapping to accelerate the slow mapping process of BCQ.
We further remove the extra deployment cost caused from the unification by integrating the two-step inference process into a single step with unification theorem (see Theorem~\ref{thm:unification}).
As a result, \methodf exhibits the best performance
without any additional memory and computational costs at deployment.

Our main contributions are as follows:
%
\begin{itemize}[leftmargin=3mm, itemsep=-1mm, topsep=-1mm]
    \item \textbf{Algorithm.}
    We propose \methodf, an accurate quantization method for LLMs.
    \methodf unifies the best-performing UQ and BCQ methods to leverage both of their advantages.
    \item \textbf{Analysis.}
    We analyze that \methodf exhibits both strong expressiveness and optimizability.

    \item \textbf{Experiments.}
    We show that \methodf outperforms UQ and BCQ methods on various benchmarks, showing up to 4.60\% higher accuracy.
\end{itemize}
\vspace{1mm}
{The source code for implementing \methodf is publicly available at \url{https://github.com/snudm-starlab/UniQuanF}.}

The rest of this paper is organized as follows.
Section~\ref{sec:prelim} defines LLM quantization problem and provides preliminaries.
We propose \methodf in Section~\ref{sec:proposed}.
Section~\ref{sec:exp} presents the experimental results, followed by a discussion of related works in Section~\ref{sec:related}.
Finally, we conclude the paper with a summary of our findings.

\section{Preliminary}
\label{sec:prelim}
%
We introduce the LLM quantization problem and describe preliminaries.
We describe the frequently used terminologies in Appendix~\ref{app:term}.
%
\begin{figure}[t]
    \centering
    \includegraphics[width=0.49\textwidth]{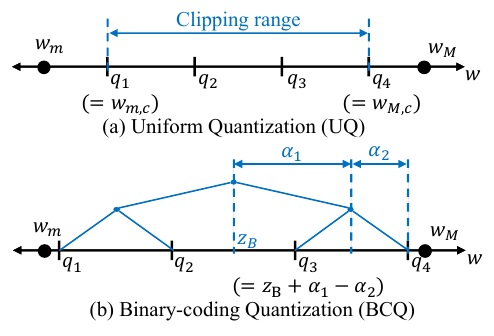}
    \caption{
	Illustrations of quantization levels ($q$) assigned for a weight group $\vw$ under (a) UQ and (b) BCQ schemes, where $w_m$ and $w_M$ are the minimum and the maximum weights in $\vw$, respectively.
    UQ has evenly-spaced quantization levels within a clipping range while BCQ has non-uniform quantization levels determined by its scale factors $\valpha$ and a shifting factor $z_B$.
    }
    \label{fig:levels}
    \vspace{-2.5mm}
\end{figure} 
\begin{figure*}[t]
    \centering
    \includegraphics[width=1.0\textwidth]{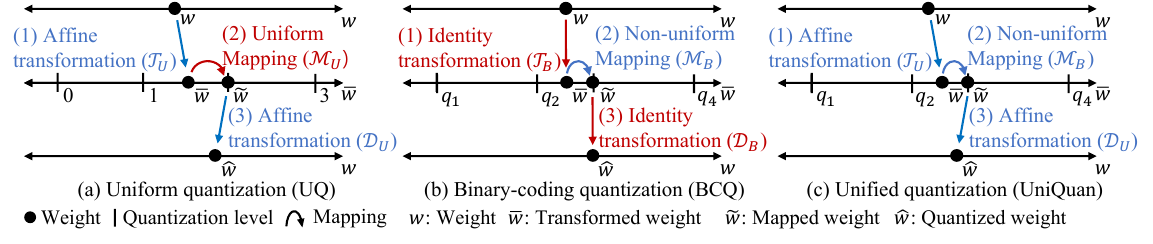}
    \vspace{-6mm}
    \caption{
    Quantization processes for a weight $w$ in (a) UQ, (b) BCQ, and (c) \method schemes.
    Blue-colored processes are parameterized functions which are the source of UQ's strong optimizability and BCQ's powerful expressiveness, respectively.
    \method takes the advantages of both UQ and BCQ schemes by combining the parameterized functions in both schemes.
    See Section~\ref{subsec:uni} for details.
    }
    \label{fig:schemes}
    \vspace{-2mm}
\end{figure*} 
\subsection{LLM Quantization Problem}
We have a pre-trained LLM $F$, a desired bit-width $k$, and a sample dataset $\sD$.
Our objective is to find an accurate $k$-bit quantized model $\widehat{F}$.
The target model is a Transformer-based~\cite{transformer} LLM~\cite{mistral,llama3} with $N$ blocks, where quantization is applied to the weight matrices in each block.
We divide each weight matrix into multiple weight groups that share quantization parameters.
Given a group $\vw\in \sR^{g}$ of $g$ weight values and a desired bit-width $k$,
a quantizer $Q$ quantizes $\vw$ into $\widehat{\vw}=Q(\vw,k;\Theta)$ where $\Theta$ is a set of quantization parameters.
%
In this paper, we focus on two quantization schemes: uniform quantization (UQ) and binary-coding quantization (BCQ) which are supported by the state-of-the-art inference kernel~\cite{lut_gemm}.
%
We directly compare their accuracy since they require the same memory and computational costs. 

\subsection{Uniform Quantization (UQ)}
\label{subsec:UQ}
Uniform quantization (UQ)~\cite{flexround,omni} is a quantization scheme that has uniformly spaced quantization levels.
We explain UQ using
Round-to-nearest (RTN)~\cite{rtn}, a representative method of UQ.
Figure~\ref{fig:levels}(a) illustrates the quantization levels in UQ that are evenly distributed in a clipping range $[w_{m,c}, w_{M,c}]$,
where $w_{m,c}$ and $w_{M,c}$ are the minimum and maximum values in the clipping range, respectively.
RTN begins with a grid search to find the proper $w_{m,c}$ and $w_{M,c}$ to precisely approximate the original weights.
The quantization parameters $\Theta_R=\{\Delta, z_{U}\}$ of RTN are derived based on its clipping range;
scale factor $\Delta = (w_{M,c} - w_{m,c})/ (2^{k}-1)$ and zero-point $z_{U} = \lfloor-w_{m,c}/\Delta\rceil$, where
$\lfloor \cdot \rceil$ is the rounding function.
Then, RTN quantizer $Q_R$ quantizes $\vw \in \sR^{g}$ into $\widehat{\vw} = Q_R(\vw,k;\Theta_R)$ as follows:
%
\begin{equation*}
\begin{split}
\widetilde{\vw} & = \textit{Clip} \left(\lfloor
\vw/\Delta + z_U \bm{1}_g \rceil, 0, 2^k-1
\right), \\
\widehat{\vw} & =\Delta (\widetilde{\vw}- z_{U}\bm{1}_g),
\end{split}
\end{equation*}
where $\textit{Clip}(\cdot, m, M)$ is an element-wise clipping function with a min-max range $[m, M]$, and $\widetilde{\vw}$ is a vector of low-bit integers assigned for each weight.
$\bm{1}_g$ is a vector of size $g$ filled with ones.

\subsection{Binary-coding Quantization (BCQ)}
\label{subsec:BCQ}
Binary-coding quantization (BCQ) is a non-uniform quantization scheme with adaptive quantization levels.
Its quantization parameters $\Theta_B=\{\valpha, z_{B}\}$ consists of a vector $\valpha\in \sR^{k}$ of scale factors and a shifting factor $z_{B}\in\sR$.
Figure~\ref{fig:levels}(b) illustrates the quantization levels in BCQ that are determined through the summation and subtraction of scale factors $\valpha$ after shifting with $z_{B}$.
In this example, BCQ quantizer defines the set
$\{
z_{B} + \alpha_{1} + \alpha_{2},
z_{B} + \alpha_{1} - \alpha_{2},
z_{B} - \alpha_{1} + \alpha_{2},
z_{B} - \alpha_{1} - \alpha_{2}
\}$
of four quantization levels forming a binary tree centered at $z_B$ with widths $\alpha_1$ and $\alpha_2$.

We assign a binary code $\vc \in \{-1, +1\}^{k}$ for each weight $w$, indicating whether the weight selects the positive or negative value for each scale factor in $\valpha$.
Then, BCQ quantizer $Q_B$ quantizes a weight group $\vw \in \sR^{g}$ into $\widehat{\vw}=Q_B(\vw;\Theta_B)$ as follows:
%
\begin{align*}
    &\widehat{\vw} =\mC\valpha + z_{B}\bm{1}_g, \\
   where \;\; \mC =& \argmin_{\mC'} || \vw - (\mC'\valpha + z_{B}\bm{1}_g) ||_2^2.
\end{align*}
$\mC\in\{-1,+1\}^{g\times k}$ is the binary code matrix for $\vw$ where its $i$th row contains the binary code for the $i$th weight in $\vw$.
Existing works~\cite{xu} first set $z_B$ as 0, then search the scale factors $\valpha$ and binary codes $\mC$ through an alternating update process (see Algorithm 2 in \citet{xu}).

The main advantage of BCQ is its strong expressiveness; any UQ is representable in the form of BCQ (see Appendix C in~\citet{lut_gemm}).
However, BCQ methods are far less accurate in quantizing LLMs than UQ methods due to the lack of accurate BCQ-based quantization methods.
In this paper, we propose \methodf which transfers UQ's accurate quantization methods to BCQ to fully exploit BCQ's strong expressiveness. 


\section{Proposed Method}
\label{sec:proposed}
\setlength{\tabcolsep}{3pt}
\begin{table*}[t]\centering
\caption{
Comparison of quantization methods.
\methodf incorporates the advantages of FlexRound and \alternating by exploiting $\ftrans$, $\detrans_R$, and $\lpgmap$, overcoming the challenges in \naivef. 
$\circ$ is a function composition operator. 
See Sections~\ref{subsec:challenges} for details.
}\label{tab:schemes}
\begin{tabular}{ccccc}
\toprule
\multicolumn{1}{c}{\multirow{2}[3]{*}{\textbf{Method}}} &
  \multicolumn{2}{c}{\textbf{Optimization}} &
  \multicolumn{2}{c}{\textbf{Deployment}} \\
  \cmidrule(lr){2-5}
\multicolumn{1}{c}{} &
  \multicolumn{1}{c}{\textbf{Initialization}} &
  \multicolumn{1}{c}{\textbf{Quantization}} &
  \multicolumn{1}{c}{\textbf{Parameters}} &
  \multicolumn{1}{c}{\textbf{Inference}} \\ \midrule
FlexRound
&  Grid search
& $\detrans_R \circ \umap \circ \ftrans$
& $\widetilde{\vw}$, $\Theta_R$
& $\detrans_R(\widetilde{\vw};\Theta_R)$
\\
\alternating
&  Alternating update
& $\gmap$
& $\mC$, $\Theta_B$
& $\recon(\mC;\Theta_B)$
\\
\naivef
&  -
& $\detrans_R \circ \gmap \circ \ftrans$
&  $\mC$, $\Theta_B$, $\Theta_R$
&  $\detrans_R(\recon(\mC;\Theta_B);\Theta_R)$
\\
\textbf{\methodf (Proposed)}
&  Unified initialization
&  $\detrans_R \circ \lpgmap \circ \ftrans$
&  $\mC$, $\Theta^*_B$
&  $\recon(\mC;\Theta^*_B)$
\\ \bottomrule
\end{tabular}
\vspace{1mm}
\end{table*}
\setlength{\tabcolsep}{6pt}

We propose \textbf{\methodf} (\methodffullbold), an accurate quantization method for LLMs.
We first propose \textbf{\method} (\methodfullbold), a framework unifying the quantization processes of UQ and BCQ schemes to exploit their strong expressiveness and optimizability.
Then, we present \naivef, a na\"ive method that unifies FlexRound and \alternating following \method, but with several limitations. 
Finally, we propose \methodf which achieves high accuracy without requiring extra costs at deployment, by addressing the limitations in \naivef.

\subsection{Unification of UQ and BCQ}
\label{subsec:uni}

We define a general representation of the quantization processes in UQ and BCQ schemes to clarify the source of their superior capabilities.
Given a quantizer $Q$ with quantization parameters $\Theta$, the general representation is defined as follows:
%
\begin{equation*}
    \widehat{\vw} = Q(\vw; \Theta) = \mathcal{D}(\mathcal{M}(\mathcal{T}(w;\Theta);\Theta);\Theta),
\end{equation*}
where $\mathcal{T}$ is a transformation process which transforms weights $\vw$ to transformed weights $\bar{\vw}$, $\map$ is a mapping process which maps transformed weights $\bar{\vw}$ to the corresponding quantization levels, and $\detrans$ is a detransformation process which reverts mapped weights $\widetilde{\vw}$ to the original weight space, yielding quantized weights $\widehat{\vw}$.
We illustrate the quantization processes in diverse quantization schemes in Figure~\ref{fig:schemes} following this general representation.

\smallsection{Optimizability of UQ}
As shown in Figure~\ref{fig:schemes}(a), we formulate the quantizer $Q_U$ of UQ as follows:
%
\begin{equation*}
    Q_U(\vw,k;\Theta_U)=\udetrans(\umap(\utrans(\vw;\Theta_U),k);\Theta_U),
\end{equation*}
where $\umap(\vw,k)=Clip(\lfloor \vw \rceil,0,2^k-1)$ is a uniform mapping function that maps transformed weights $\widetilde{\vw}$ to uniform quantization levels.
$\utrans$ and $\udetrans$ are affine transformation functions with quantization parameters $\Theta_U$.
%
There are two main methods belonging to UQ: RTN and FlexRound.
In RTN, the quantizer $Q_R$
is formulated as follows:
%
\begin{equation*}
    Q_R(\vw,k;\Theta_R)=\detrans_R(\umap(\trans_R(\vw;\Theta_R),k);\Theta_R),
\end{equation*}
where $\Theta_R=\{\Delta, z_{U}\}$, $\trans_R(\vw;\Theta_R)=\vw/\Delta+z_U\bm{1}_g$ and $\detrans_R(\vw;\Theta_R)=\Delta(\vw-z_U\bm{1}_g)$.
In FlexRound,
the main idea is to allow each weight to explore diverse quantization levels to find the optimal one.
This is achieved by replacing the transformation function $\trans_R$ with $\ftrans$ defined as follows:
%
\begin{equation}
    \ftrans(\vw;\Theta_F) = \vw\oslash(\Delta \cdot \vs \cdot s_r) + z_U\bm{1}_g, \label{eq:ftrans}
\end{equation}
where $\Theta_F=\{\Delta, z_U, \vs, s_r\}$ is a set of quantization parameters of FlexRound, and $\oslash$ denotes an element-wise division.
$\vs\in\sR^{g}$ and $s_r\in\sR$ are element-wise and row-wise scale factors to promote the exploration of weights, respectively.

In summary, each quantization method has its own $\Theta_U$, $\trans_U$, and $\detrans_U$,
and they are the origin of strong optimizability. 
We formulate the quantizers of various UQ methods in Appendix~\ref{app:sec:extension}.

\smallsection{Expressiveness of BCQ}
As shown in Figure~\ref{fig:schemes}(b), we formulate the quantizer $Q_B$ of BCQ as follows:
%
\begin{equation*}
    Q_B(\vw;\Theta_B)=\bdetrans(\gmap(\btrans(\vw);\Theta_B)),
\end{equation*}
where $\btrans(\vw)=\bdetrans(\vw)=\vw$ are identity functions.
$\gmap(\vw;\Theta_B)=Q_B(\vw;\Theta_B)$ is a non-uniform mapping function that maps transformed weights $\widetilde{\vw}$ to the nearest non-uniform quantization levels.
$\gmap$ enables the mapping between weights and non-uniform quantization levels, enabling strong expressiveness.
\alternating~\cite{xu} is the only quantization method in BCQ scheme applicable for LLMs with the same quantizer $Q_B$ defined above.

\smallsection{\method (Unified Quantization)}
We propose \method, a framework unifying the quantization processes of UQ and BCQ schemes.
\method formulates the unified quantizer $Q_I$ with quantization parameters $\Theta_I=\{\Theta_B, \Theta_U\}$ as follows (see Figure~\ref{fig:schemes}(c)):
%
\begin{equation*} \label{eq:qi}
    Q_I(\vw;\Theta_I) = \udetrans(\gmap(\utrans(\vw;\Theta_U);\Theta_B);\Theta_U).
\end{equation*}
%
$Q_I$ has strong expressiveness by mapping weights to the non-uniform quantization levels using $\gmap$.
Also, $Q_I$ has strong optimizability by incorporating useful methods in UQ by replacing $\Theta_U$, $\trans_U$, and $\detrans_U$ with those of the method to incorporate.
We elaborate on incorporating FlexRound~\cite{flexround}, the best-performing UQ method, into BCQ in Sections~\ref{subsec:challenges} and~\ref{sec:overview}.
We cover the incorporation of AWQ~\cite{awq}, OmniQuant~\cite{omni}, and GPTQ~\cite{optq} into BCQ in Section~\ref{app:sec:extension}.

\subsection{\methodfz}
\label{subsec:challenges}
We introduce a basic method \naivef which naively unifies the best-performing UQ and BCQ methods, FlexRound~\cite{flexround} and \alternating~\cite{xu}, respectively.
The quantizer $Q_{I_F}$ of \naivef is defined as follows:
%
\begin{equation} \label{eq:qif}
\begin{split}
    Q_{I_F}&(\vw;\Theta_{I_F}) \\
    & = \detrans_R(\gmap(\ftrans(\vw;\Theta_F);\Theta_B);\Theta_R).
\end{split}
\end{equation}

Note that FlexRound begins its optimization process by initializing its quantization parameters with a grid search process to find a proper clipping range.
FlexRound enhances its optimizability by utilizing $\ftrans$ in Equation~\ref{eq:ftrans} and optimizes its quantization parameters $\Theta_F$ using stochastic gradient descent (SGD).
After optimization, FlexRound stores low-bit weights $\widetilde{\vw}=\umap(\ftrans(\vw;\Theta_F),k)$ and quantization parameters $\Theta_R=\{\Delta, z_U\}$, and approximates the original weights $\vw$ as $\widehat{\vw}=\detrans_R(\widetilde{\vw};\Theta_R)$ in the detransformation step; $\vs$ and $s_r$ in $\Theta_F$ are used only for transformation.
On the other hand,
\alternating initializes its quantization parameters in $\Theta_B=\{\valpha, z_B\}$ with alternating update~\cite{xu}.
Then, it stores the binary code matrix $\mC$ and quantization parameters $\Theta_B$.
It approximates weights $\vw$ as $\widehat{\vw}=\mathcal{R_B}(\mC;\Theta_B) = \mC\valpha + z_{B}\bm{1}_g$ for inference.

\naivef inherits the strong optimizability from FlexRound's $\ftrans$ and $\detrans_R$, and the strong expressiveness from \alternating's $\gmap$.
However, \naivef has three points of improvements as follows: 
\begin{itemize}[leftmargin=7mm, itemsep=-1mm, topsep=-1mm]
\item[\textbf{C1.}] \textbf{Joint initialization.}
FlexRound and \alternating are based on different initialization methods which seem incompatible.
How can we jointly initialize the quantization parameters considering their dependency?
\item[\textbf{C2.}] \textbf{Slow mapping process.}
The quantization levels of \naivef are changed during optimization process where its quantization parameters are updated.
Therefore, it is crucial to update the mappings between weights and quantization levels during optimization using $\gmap$.
However, existing method~\cite{xu} exploits a binary search tree (BST) to implement $\gmap$ which finds the mappings through a GPU-unfriendly sequential process, and it requires intractable time for updating mappings during optimization.
How can we expedite the mapping process?
%
\item[\textbf{C3.}] \textbf{Expensive deployment cost.}
\naivef suffers from memory and computational overheads at deployment, since \naivef require executing both inference procedures $\detrans_R$ and $\recon$ (see Table~\ref{tab:schemes}).
How can we decrease the inference cost?
\end{itemize}
%



\begin{algorithm}[t]
    \caption{\methodf for a block $f$}\label{alg:uniquan}
    \begin{algorithmic}[1]
        \item[\textbf{Input:}]  
         A Transformer block $f$, $f$'s input $\mX$, a set $\Phi$ of weight groups in $f$, the input $\widehat{\mX}$ of a quantized $f$ obtained by preceding quantized blocks, number $E$ of epochs, and number $|\sD|$ of data points
        \item[\textbf{Output:}] 
        Optimized quantization parameters $\Theta^*_{B}$ and a binary-code matrix $\mC$ for each group
        \STATE Initialize $\Theta_{I_F}$ for each weight group
        \\ \hfill $\triangleright$ \textbf{(I1)} Unified Initialization
        \FOR {epoch $e$ in 1, 2, ..., $E$}
        \FOR {$s$ in 1, 2, ..., $|\sD|$}
        \STATE Quantize $\Phi$ into $\widehat{\Phi}$ using Equation~\ref{eq:qif}
        \\ \hfill $\triangleright$ \textbf{(I2)} Local and periodic mapping
        \STATE
        $\mathcal{L}$ $\gets$ $|| f(\mX_{s};\Phi) - f(\widehat{\mX}_{s}; \widehat{\Phi}) ||_F^2$
        \STATE Perform backpropagation using $\mathcal{L}$
        \STATE Update $\Theta_{I_F}$ for each group
        \ENDFOR
        \ENDFOR
        \STATE Convert each $\Theta_{I_F}$ into $\Theta^*_{B}=\{\valpha^*, z_B^*\}$
        \\ \hfill $\triangleright$ \textbf{(I3)} Unification Theorem
        \STATE Find $\mC$ using Equation~\ref{eq:get_C} for each group
    \end{algorithmic}
\end{algorithm}

\subsection{\methodf}
\label{sec:overview}

We propose \textbf{\methodf} (\methodffullbold),
which improves \naivef with the following main ideas.
%
\begin{itemize}[leftmargin=6mm, itemsep=-1mm, topsep=-1mm]
\item[\textbf{I1.}]
\textbf{Unified initialization.}
We unify the initialization algorithms in FlexRound and \alternating to jointly initialize the quantization parameters, considering their dependency.
%
%
\item[\textbf{I2.}]
\textbf{Local and periodic mapping.}
We efficiently update mappings by exploiting the locality and the temporal sparsity of changes in mappings. 
%
\item[\textbf{I3.}] \textbf{Unification theorem.}
We unify the inference process of \naivef into a single BCQ's inference process, removing additional memory and computational costs at deployment.
\end{itemize}
\vspace{1mm}

\methodf quantizes an LLM through a block-wise optimization strategy, quantizing sequentially from the bottom-most Transformer block upwards.
Algorithm~\ref{alg:uniquan} details the process of quantizing a Transformer block $f$ with a set $\Phi$ of weight groups, to obtain a binary-code $\mC$ and optimized quantization parameters $\Theta_B^*$ for each group.
We provide inputs $\mX$ and $\widehat{\mX}$ of unquantized and quantized blocks, respectively.
We iterate the optimization process for $E$ epochs using $|\sD|$ data points.

First, for each weight group, we initialize the quantization parameters $\Theta_{I_F}=\Theta_B \cup \Theta_F$ using the unified initialization technique (line 1) which we describe later in this subsection.
Subsequently, we iteratively optimize the quantization parameters to minimize the distance between the outputs of the quantized and unquantized blocks (lines 2-9).
In each optimization step, we quantize the weight groups in $\Phi$ using \methodf's quantizer $Q^*_{I_F}$ with its quantization parameters $\Theta_{I_F}$ as follows:
%
\begin{equation} \label{eq:qifhat}
\begin{split}
    Q^*_{I_F}&(\vw;\Theta_{I_F}) \\
    & = \detrans_R(\lpgmap(\ftrans(\vw;\Theta_F);\Theta_B);\Theta_R),
\end{split}
\end{equation}
where $\lpgmap$ represents the local and periodic mapping technique, which we describe later in this subsection, to accelerate mappings to non-uniform quantization levels.
$\ftrans$ is the transformation function of FlexRound in Equation~\ref{eq:ftrans}.
We quantize each weight group in parallel, and the quantized weight groups are stored in $\widehat{\Phi}$ (line 4).
We measure the reconstruction loss $\mathcal{L}$ which represents the distance between the outputs of the original and quantized blocks, and then optimize the quantization parameters $\Theta_{I_F}$ via stochastic gradient descent (lines 5-7).
After the iterative optimization process, we have the optimized quantization parameters $\Theta_{I_F}=\Theta_B \cup \Theta_F$ which generate the outputs closely resembling those of the original block.
We leverage the unification theorem (Theorem~\ref{thm:unification}) to merge the optimized quantization parameters into a single set of BCQ's quantization parameters $\Theta^*_B=\{\valpha^*, z_B^*\}$ (line 10).
Finally, we obtain a binary-code matrix $\mC$ using $\Theta_B^*$ for each weight group as follows:
\begin{equation} \label{eq:get_C}
    \mC=\argmin_{\mC'} || \vw - (\mC'\valpha^* + z^*_{B}\bm{1}_g) ||_2^2.
\end{equation}
Then, the quantization process for a block is completed (line 11).
We obtain an entirely quantized model by sequentially applying Algorithm~\ref{alg:uniquan} from the bottom to the top blocks.

\smallsection{Unified Initialization}
How can we jointly initialize the quantization parameters of FlexRound and \alternating?
FlexRound initializes $\Delta$ and $z_U$ through a grid search process, and \alternating utilizes an alternating update process to find $\valpha$, fixing $z_B$ as 0.
We propose a unified initialization process which initializes the quantization parameters $\Theta_{I_F}=\{ \Delta, z_U, \vs, s_r, \valpha, z_B \}$ by unifying both techniques, considering their dependency.

In unified initialization, we perform alternating updates to find $\valpha$ and $z_B$ in each iteration of a grid search process for $\Delta$ and $z_U$.
As a result, we find $\valpha$ and $z_B$ according to the $\Delta$ and $z_U$ during the grid search process.
As in FlexRound, we initialize $\vs$ and $s_r$ as $\bm{1}_g$ and 1, respectively.
We detail the unified initialization algorithm in Appendix~\ref{subsec:A_ui}.

\setlength{\tabcolsep}{3pt}

\setlength{\tabcolsep}{3pt}
\begin{table}[t] \centering
\caption{
The percentage of index changes per optimization step and throughout all 2,560 steps.
Minimal changes in each step add up to substantial shifts overall.
}\label{tab:idf}\vspace{-1mm}
\begin{tabular}{ccccc}
\toprule
\multirow{2}[3]{*}{\textbf{Step size}} & \multicolumn{4}{c}{\textbf{Index change}}      \\
\cmidrule(lr){2-5}
                                    & \textbf{0} & \textbf{1} & \textbf{2}
                                     & \textbf{$>$2} \\ \midrule
Single step      & 99.97\%     & 0.01\%      & 0.00\%      & 0.01\% \\
All steps (2,560) & 89.30\%     & 8.37\%      & 1.36\%      & 0.97\% \\ \bottomrule
\end{tabular}
\end{table}
\setlength{\tabcolsep}{6pt} 

\smallsection{Local and Periodic Mapping}
\label{sec:i2}
How can the mappings between weights and quantization levels be swiftly updated during optimization?
%
Table~\ref{tab:idf} shows that while over 8\% of weights change indices of the mapped quantization levels during the full optimization, 99.97\% remain stable at each step, and most changes involve adjacent levels.
Hence, we propose a local and periodic mapping which evaluates only neighboring levels and remaps periodically.
We set a hyperparameter $p$ as a remapping period, and in every $p$ steps, we find the closest quantization level among the previous level and its two neighbors.
We detail the process of local and periodic mapping in Appendix~\ref{alg:lpmapping}.
\begin{table*}[t]
\caption{
Comparison of the average accuracy of 0-shot and 5-shot on MMLU, and the perplexity on WikiText2 (Wiki) benchmarks.
Higher accuracies and lower perplexities indicate better performance.
Bold and underlined texts indicate the best and second-best performance, respectively.
\methodf shows the best performance in most cases.
}
\label{tab:foundation}
\centering
\begin{tabular}{cclcccccc}
\toprule
&
 &
 \multicolumn{1}{c}{} &
 \multicolumn{2}{c}{\textbf{Mistral 7B}} &
 \multicolumn{2}{c}{\textbf{Llama-3 8B}} &
 \multicolumn{2}{c}{\textbf{Llama-3 70B}} \\
&
 &
 \multicolumn{1}{c}{} &
 &
 &
 &
 &
 &
 \\
\multirow{-3}{*}{\textbf{\# Bits}} &
 \multirow{-3}{*}{\textbf{Scheme}$^\dagger$} &
 \multicolumn{1}{c}{\multirow{-3}{*}{\textbf{Method}}} &
 \multirow{-2}{*}{\begin{tabular}[c]{@{}c@{}}MMLU\\ Avg. ($\uparrow$)\end{tabular}} &
 \multirow{-2}{*}{\begin{tabular}[c]{@{}c@{}}Wiki\\ ($\downarrow$)\end{tabular}} &
 \multirow{-2}{*}{\begin{tabular}[c]{@{}c@{}}MMLU\\ Avg. ($\uparrow$)\end{tabular}} &
 \multirow{-2}{*}{\begin{tabular}[c]{@{}c@{}}Wiki\\ ($\downarrow$)\end{tabular}} &
 \multirow{-2}{*}{\begin{tabular}[c]{@{}c@{}}MMLU\\ Avg. ($\uparrow$)\end{tabular}} &
 \multirow{-2}{*}{\begin{tabular}[c]{@{}c@{}}Wiki\\ ($\downarrow$)\end{tabular}} \\
\midrule
\multicolumn{3}{c}{Full precision} &
 61.40 &
 5.25 &
 62.77 &
 6.14 &
 77.52 &
 2.86 \\
\midrule
&
 &
 RTN &
 55.20 &
 6.04 &
 54.67 &
 7.80 &
 67.74 &
 3.85 \\
&
 &
 OmniQuant &
 56.28 &
 5.67 &
 56.02 &
 {\underline{7.37}} &
 76.62 &
 3.34 \\
&
 \multirow{-3}{*}{UQ} &
 FlexRound &
 {\underline{58.60}} &
 {\underline{5.52}} &
 {\underline{58.59}} &
 8.22 &
 \textbf{77.24} &
 \underline{3.31} \\
\cdashline{2-9}
&
 &
\alternating &
 24.65 &
 9.8e4 &
 22.95 &
 1.3e5 &
 44.07 &
 11.06 \\
\multirow{-5}{*}{4} &
 \multirow{-2}{*}{BCQ} &
\textbf{\methodf} &
\textbf{59.22} &
\textbf{5.51} &
{\textbf{61.43}} &
{\textbf{7.01}} &
{\underline{76.97}} &
\textbf{3.19} \\

\midrule
&
 &
 RTN &
 28.59 &
 36.07 &
 24.35 &
 120.50 &
 40.61 &
 24.42 \\
&
 &
 OmniQuant &
 32.82 &
 9.24 &
 26.34 &
 42.06 &
 71.13 &
 \underline{5.23} \\
&
 \multirow{-3}{*}{UQ} &
 FlexRound &
 {\underline{53.45}} &
 {\underline{6.41}} &
 {\underline{50.22}} &
 \underline{10.11} &
 \underline{72.44} &
 5.81 \\
\cdashline{2-9}
&
 &
 \alternating &
 24.19 &
 1.1e4 &
 25.51 &
 8.4e4 &
 24.06 &
 1.5e3 \\
\multirow{-5}{*}{3} &
 \multirow{-2}{*}{BCQ} &
 \textbf{\methodf} &
 \textbf{53.68} &
 \textbf{6.20} &
 {\textbf{53.46}} &
 {{\textbf{8.75}}} &
 \textbf{74.79} &
 \textbf{4.24}
\\ \bottomrule
\multicolumn{9}{l}{$\dagger$ UQ and BCQ methods have the same costs when using BCQ kernels~\cite{lut_gemm}}
\end{tabular}
\end{table*}
\smallsection{Unification Theorem}
How can we eliminate the memory and computational overheads at deployment induced by the unification? 
    The deployment cost of \method is expensive since the unified quantization process in Equation~\ref{eq:qifhat} requires two-step inference procedures with $\detrans_R$ and $\recon$.
We propose a unification theorem that integrates the two-step inference process of \naivef into a single BCQ inference process as follows:

\begin{theorem}[Unification Theorem] \label{thm:unification}
    Given
    a reconstruction function $\mathcal{R}_B(\mC;\Theta_B)$
    and
    a detransformation function $\mathcal{D}_R(\widetilde{\vw};\Theta_R)$ where $\widetilde{\vw} = \mathcal{R}_B(\mC;\Theta_B)$,
    there is a set $\Theta_B^*$ of unified quantization parameters such that $\mathcal{R}_B(\mC;\Theta_B^*) =
    \mathcal{D}_R(\mathcal{R}_B(\mC;\Theta_B);\Theta_R)$
    for any $\mC$, $\Theta_B$, and $\Theta_R$.
\end{theorem}
%
\begin{proof}
     See Appendix~\ref{app:proof}.
\end{proof}
As a result, \methodf exploits the UQ's optimizability without any extra memory and computational costs at deployment.


\smallsection{Discussion}
\methodf is interpreted as a quantization framework that augments the BCQ's quantization process with the linear transformation process of UQ.
By defining the transformation function in \methodf as a linear function, several key advantages emerge.
First, it enables the seamless integration of well-established UQ methods~\cite{flexround,awq,omni} into the BCQ framework, thereby reducing the time and effort required to develop new methods.
Second, through the unification theorem, the inference process is consolidated into a single step, allowing the system to benefit from the improved accuracy of the unified approach without incurring additional memory or computational costs at deployment. One might consider introducing non-linear transformation functions to further enhance the expressiveness of the unified quantization process.
However, this approach presents notable drawbacks since the aforementioned two key advantages do not hold for non-linear transformation functions; we cannot exploit existing UQ methods and unification theorem.
Therefore, we incorporate linear transformation processes into the BCQ's quantization process, rather than non-linear transformation processes, to enhance efficiencies in method development and deployment.

\section{Experiments}
\label{sec:exp}
We perform experiments to answer the following questions.
Additional analyses on \methodf are discussed in Appendix~\ref{app:analysis}.

\begin{itemize*}
    \item[\textbf{Q1.}] \textbf{General knowledge evaluation.}
    How accurately does \methodf quantize LLMs on general knowledge benchmarks?
    \item[\textbf{Q2.}] \textbf{Task-specific knowledge evaluation.}
    How accurately does \methodf quantize LLMs on task-specific knowledge benchmarks?
    \item[\textbf{Q3.}] \textbf{Ablation study.}
    Do all ideas in \methodf improve the accuracy of quantized LLMs?
    \item[\textbf{Q4.}] \textbf{Case study.}
    Does \methodf effectively utilize BCQ's expressiveness and UQ's optimizability as we intended?
\end{itemize*}

\subsection{Experimental Setup} \label{subsec:setup}
We briefly introduce the experimental setup.
Refer to Appendix~\ref{app:impl} for further details.

\smallsection{Setup}
We evaluate the quantized performance of Llama-3 8B, Llama-3 70B~\cite{llama3}, and Mistral 7B~\cite{mistral} models on MMLU~\cite{mmlu} and WikiText2~\cite{wiki} benchmarks.
For task-specific experiment, we quantize Llama-3 8B Instruct~\cite{llama3} model and evaluate on GSM8K~\cite{gsm8k}.
We sample 128 token sequences of length 2048 from C4~\cite{c4} and GSM8K~\cite{gsm8k} for general and task-specific benchmarks, respectively.
All experiments are done with a single A100 GPU.

\smallsection{Baselines}
We compare \methodf with the four UQ and BCQ methods: RTN~\cite{rtn}, OmniQuant~\cite{omni}, FlexRound~\cite{flexround}, and \alternating~\cite{xu}.
%

\smallsection{Hyperparameters}
%
We set each row of weight matrices as a single weight group except for Llama-3 70B where we divide each row into small weight groups of size 128 to prevent severe accuracy loss. Note that each row has at least 1024 weights.
Hyperparameters of \methodf is in Appendix~\ref{app:subsec:hyps}.
We provide a thorough and detailed analysis of \methodf by altering its hyperparameters to guide hyperparameter search processes in Section~\ref{app:analysis}.
%

\subsection{General Knowledge Evaluation}
%
We evaluate the amount of general knowledge retained in quantized models.
Table~\ref{tab:foundation} summarizes the average accuracies of 5-shot and 0-shot MMLU, and perplexity on WikiText2 (Wiki).
\methodf achieves the highest performance in almost all cases,  
outperforming the second-best method by up to 3.24\% in average MMLU accuracy.
In Llama-3 70B, a small group size of 128
makes a small number of weights sharing their quantization parameters,
reducing the difficulty of quantization; competitors show comparative performance to \methodf in both 3- and 4-bit cases.
%


\begin{table}[t]
    \centering
      \caption{
      Accuracies (\%) of quantized
      Llama-3 Instruct 8B models on GSM8K benchmark.
      \methodf outperforms the competitors in all cases.
      }
      \setlength{\tabcolsep}{9.5pt}
      \label{tab:finetuned}
      \centering
      \begin{tabular}{llcc}
      \toprule
      \textbf{Scheme} & \textbf{Method} & \textbf{4bit} & \textbf{3bit} \\
      \midrule
      \multicolumn{2}{c}{Full precision} & \multicolumn{2}{c}{76.12} \\
      \midrule
      \multirow{3}{*}{UQ} & RTN       & 49.43  & 0.08  \\
  & OmniQuant & 59.16    & 2.78 \\
  & FlexRound & \underline{70.66} & \underline{54.13} \\
      \hdashline
      \multirow{2}{*}{BCQ} & \alternating & 0.00 & 0.00  \\
   & \textbf{\methodf}
   & \textbf{72.38}
   & \textbf{58.73} \\
      \bottomrule
\end{tabular}
\end{table}
  \setlength{\tabcolsep}{6pt}


%

\subsection{Task-specific Knowledge Evaluation}
%
Table~\ref{tab:finetuned} summarizes the accuracies of a quantized Llama-3 8B Instruct model evaluated on GSM8K benchmark, which focuses exclusively on mathematical problems.
Note that GSM8K is a difficult task in which models should generate the exact answer, while general knowledge benchmarks require only selecting correct answers among multiple choices, or evaluating given texts.
\methodf achieves the highest performance, outperforming the second-best competitor by 4.60\% and 1.72\% in the 3-bit and 4-bit experiments, respectively.
Notably, \methodf significantly improves the performance of FlexRound by utilizing \alternating even when \alternating exhibits zero accuracy.
This demonstrates that \methodf effectively unifies UQ and BCQ methods, leading to achieving previously unattainable performance.


\subsection{Ablation Study}
Table~\ref{tab:ablation} summarizes the result of an ablation study using a  Llama-3 8B model to show the effectiveness of our main ideas.
%
%
``-remapping'' refers to the case where we do not apply local and periodic mapping, thus mappings are not updated during optimization.
``-unified init. ($\Theta_F$)" and ``-unified init. ($\Theta_B$)" denote the cases when we do not use unified initialization.
In ``-unified init. ($\Theta_F$)", we do not perform a grid search process to initialize quantization parameters in $\Theta_F$, i.e. we initialize $\Delta$, $z_U$, $s$, and $s_r$ in $\Theta_F$ as 1, 0, $\bm{1}_g$, and 1, respectively.
In ``-unified init. ($\Theta_B$)", we do not perform an alternating update process to initialize quantization parameters in $\Theta_B$, i.e. we initialize $\valpha$ and $z_B$ in $\Theta_B$ to have uniformly spaced quantization levels in the space transformed by $\ftrans$.
As a result, \methodf outperforms all of its variants, showing the effectiveness of our main ideas.
Specifically, precisely initializing $\Theta_F$ is crucial to exploit the strong optimizability of $\Theta_F$.
%

\begin{table}[t]
    \centering
     \caption{
    %
    Ablation study results on MMLU benchmark.
    All ideas contribute to improving accuracies.
    }
    \setlength{\tabcolsep}{9.5pt}
    \label{tab:ablation}
    \centering
    \begin{tabular}{lcc}
    \toprule
    \textbf{Method} & \textbf{4bit} & \textbf{3bit}\\
    \midrule
    \textbf{\methodf} & \textbf{61.43} & \textbf{53.46}     \\
    - remapping  & 56.89 & 47.83       \\
    - unified init. ($\Theta_F$)  & 24.70 & 22.87   \\
    - unified init. ($\Theta_B$)  & \underline{60.84} & \underline{52.14} \\
    \bottomrule
    \end{tabular}
\end{table}
  \setlength{\tabcolsep}{6pt}


%

\subsection{Case Study}
\label{subsec:anaylsis}
%
We conduct a case study on a Llama-3 8B model quantized with \methodf to verify whether the intended mechanisms function properly.

\smallsection{Expressiveness}
Figure~\ref{fig:anal}(a) visualizes the distribution of weights (green bars) within a weight group and the assigned quantization levels (red triangles) following \methodf.
As shown in the figure, \methodf assigns dense quantization levels near zero, where most weights are concentrated, and sparse levels in regions with fewer weights.
This demonstrates that \methodf effectively learns non-linear quantization levels aligned to the weight distribution, harnessing the strong expressiveness of BCQ scheme.

\smallsection{Optimizability}
Figure~\ref{fig:anal}(b) shows the changes of reconstruction error during optimization for a block when using FlexRound (green), \alternatings (brown), and \methodf (blue); \alternatings is a modification of \alternating that optimizes $\valpha$ using a block-wise output reconstruction process in \methodf devised for comparing optimizability.
Compared to the baselines, \methodf achieves the least error at the end with a steep decrease at the early stage of optimization.
This demonstrates that \methodf effectively minimizes the reconstruction error, harnessing the strong optimizability of UQ scheme.

In summary, \methodf successfully unifies UQ and BCQ schemes, harnessing their strong expressiveness and optimizability, as we intended.
We further verify that \methodf effectively leverages flexible mapping, the main optimization technique of FlexRoud, in Appendix~\ref{subsec:app:flex}.

\begin{figure}[t]
    \centering
    \includegraphics[width=\columnwidth]{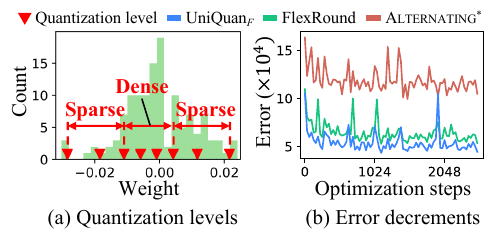}
    \caption{
    (a) Quantization levels learned by \methodf and (b) changes of reconstruction error during optimization.
    \methodf assigns quantization levels that align closely with the weight distribution, and effectively reduces the reconstruction error during optimization.
    }
    \label{fig:anal}
\end{figure}

\section{Related Work}
\label{sec:related}
\smallsection{Quantization}
Quantization~\citep{awq,quarot,sensimix,spinquant} has gained great attention since it effectively speeds up LLM inference by minimizing the cost of memory operations.
The majority of quantization approaches rely on uniform quantization (UQ)~\citep{omni, flexround} due to its hardware-friendly nature, benefiting from pre-existing acceleration kernels.
Some works~\cite{quarot,spinquant} even quantize activation by mitigating outliers through weight rotation, which is orthogonal to our approach.
Other approaches, such as vector~\citep{SqueezeLLM, quipsharp, vptq} and additive quantization~\citep{aqlm} suffer from slow inference since they lack of acceleration kernels.
Binary-coding quantization (BCQ)~\citep{xu, alphatuning} also has been underexplored due to its missing kernels even with broader expressiveness over UQ (see Section~\ref{subsec:BCQ}).
However, recent studies~\citep{lut_gemm, shiftadd} have proposed kernels leveraging look-up tables, enabling BCQ to achieve the same speed as UQ.
Taking advantage of them, \method unlocks the superior expressiveness of BCQ by unifying it with the UQ scheme.

\smallsection{Other compression methods}
Pruning~\cite{shortgpt,sparsegpt,llmpruner,sprint,kprune,blockpruner} and knowledge distillation~\cite{minillm, minilm, pet,falcon,auber,peakd,kegnet,jhkim21} are compatible techniques that improve the accuracy of compressed models when used with quantization~\cite{park_survey}.
Pruning improves the accuracy of quantized models by allowing us to utilize more bit-widths for the necessary weights under memory constraints through explicitly removing unnecessary weights.
Knowledge distillation boosts the accuracy of quantized models by transferring the useful knowledge of uncompressed models to quantized models during quantization.

\vspace{1mm}
\section{Conclusion}
\label{sec:conclusion}
\vspace{1mm}
We propose \methodf, an accurate quantization method for LLMs.
We unify the best-performing uniform quantization (UQ) and binary-coding quantization (BCQ) methods to leverage their strong optimizability and expressiveness concurrently.
We propose local and periodic mapping, unified initialization, and unification theorem to improve the accuracies of quantized models without introducing additional memory and computational costs at deployment. 
As a result, \methodf achieves the best performance, demonstrating the effectiveness of the unification.
Incorporating diverse UQ methods into BCQ, and integrating UQ methods with other non-uniform quantization schemes beyond BCQ are our promising future works. 


\newpage

\section*{Acknowledgements}
This work was supported by Institute of Information \& communications Technology Planning \& Evaluation (IITP) grant funded by the Korea government (MSIT)
[No.RS-2020-II200894, Flexible and Efficient Model Compression Method for Various Applications and Environments],
[No.RS-2021-II211343, Artificial Intelligence Graduate School Program (Seoul National University)],
and [No.RS-2021-II212068, Artificial Intelligence Innovation Hub (Artificial Intelligence Institute, Seoul National University)].
This work was supported by Youlchon Foundation.
The Institute of Engineering Research at Seoul National University provided research facilities for this work.
The ICT at Seoul National University provides research facilities for this study.
U Kang is the corresponding author.
This work was partly done while Seungcheol Park was an intern at NAVER Cloud. This work was improved by the dedicated input and collaboration of researchers from NAVER Cloud.

\section*{Limitations}
\label{sec:limitation}

In this paper, we define \method which is a quantization scheme that unifies uniform quantization (UQ) and binary-coding quantization (BCQ) schemes.
We verify the effectiveness of \method with only \methodf which unifies the best-performing UQ and BCQ methods.
We cover the unification of extended quantization methods in Appendix~\ref{app:sec:extension} in theoretical aspects.
The experimental validation of diverse pairs of quantization methods in extended quantization schemes is one of our promising future works.
%



\bibliography{reference}

\newpage

\appendix

\section{Terminology}
\label{app:term}
We summarize the definitions of terminologies frequently used in this paper to promote clarity.

\subsection{Units in LLMs}
We summarize the definitions of the units in Large Language Models (LLMs) from a weight to a model.
Figure~\ref{fig:trans} illustrates an example of a Transformer-based LLM with $N$ blocks.

\noindent

\begin{figure}[h]
  \centering
  \includegraphics[width=0.55\linewidth]{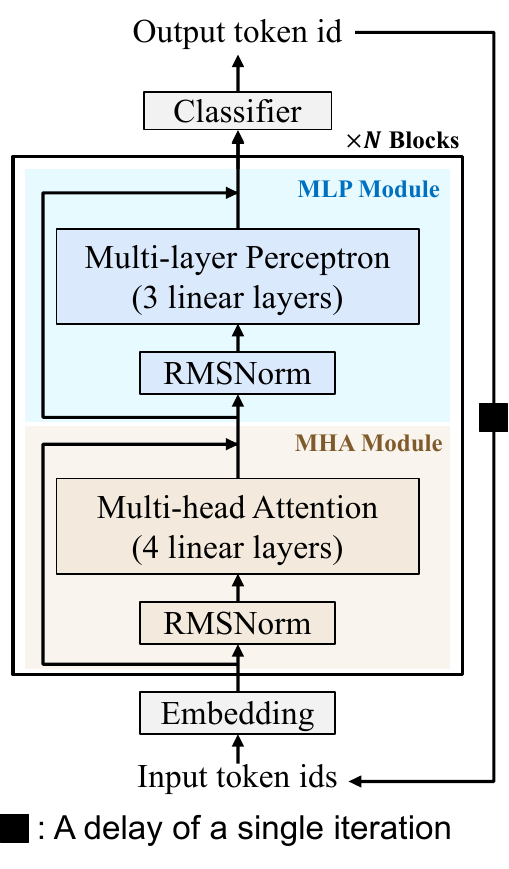} 
  \captionof{figure}{An illustration of a Transformer architecture with $L$ blocks.}
  \label{fig:trans}
\end{figure}
\begin{itemize}[leftmargin=3mm, itemsep=0mm, topsep=0mm]
    \item \textbf{Weight:} the smallest unit, representing an individual numerical weight value. The quantization bit width represents the number of bits to represent each weight.
    \item \textbf{Weight group:} a collection of weights grouped by a specified group size, all of which share the same quantization parameters.
    It is easy to quantize models when we have small group sizes since we have plenty of quantization parameters.
    \item \textbf{Weight matrix:} a two-dimensional matrix composed of weights, containing multiple weight groups.
    \item \textbf{Layer:} a component that performs affine transformations with a weight matrix and a bias vector.
    \item \textbf{Module:} a collection of layers that performs a specific functionality. In Transformers, modules include Multi-Head Attention (MHA) and Multi-Layer Perceptron (MLP).
    \item \textbf{Block:} a fundamental unit of a Transformer, consisting of one MHA module and one MLP module.
    FlexRound~\cite{flexround} and \methodf sequentially quantize each block from the bottom to top to reduce the cost of quantization.
    \item \textbf{Model:} a complete language model consisting of multiple blocks. An LLM refers to a model.
\end{itemize}



\subsection{Error Types}
In this paper, we mention two types of errors: quantization error and output reconstruction error.
Quantization error represents the error before and after quantization at the weight level without considering the model's input.
If we quantize a weight group $\vw$ into $\widehat{\vw}$, then the quantization error is $|| \vw - \widehat{\vw} ||_2^2$.
On the other hand, output reconstruction error measures how the quantized model's outputs are close to those before quantization.
We measure the reconstruction error at block level.
Assume that we have a block $f$ with parameters $\Phi$,
quantized into $\widehat{f}$ with parameters $\widehat{\Phi}$.
The reconstruction error is $||f(\mX;\Phi) - \widehat{f}(\widehat{\mX};\widehat{\Phi})||_F^2$, where $\mX$ and $\widehat{\mX}$ are the inputs of the block $f$ and the quantized block $\widehat{f}$, respectively.

We first initialize the quantization parameters to minimize quantization errors at the weight level. 
Then, we optimize the quantization parameters to minimize the output reconstruction errors at the block level to take account of the input distribution.
\setlength{\tabcolsep}{5pt}
\begin{table}[t!]
\centering
\caption{
	Symbols and their definitions.
}
\label{tab:symbols}
\renewcommand{\arraystretch}{0.9}
\resizebox{\linewidth}{!}{
\begin{tabular}{cl}
\toprule
\textbf{Symbol} & \textbf{Definition}  \\
\midrule
$w$ & Weight \\
$\vw$ & Group of weights \\
$\bar{\vw}$ & Group of transformed weights \\
$\widetilde{\vw}$ & Group of mapped weights \\
$\widehat{\vw}$  & Group of quantized weights \\
$\Phi$ & Set of weight groups \\
$g$ & The size of a weight group \\
$k$ & Quantization bit-width \\
$f$ & Transformer block \\
$|\sD|$ & Number of data points \\

\midrule

$\bm{1}_g$ & Vector of size $g$ filled with ones \\
$\lfloor \cdot \rceil$ & Rounding function \\
$Clip(\cdot,m,M)$ & Clipping function with range $[m, M]$ \\
$\oslash$ & Element-wise division \\

\midrule

$\Delta$  & Scale factor of UQ \\
$z_U$ & Zero-point of UQ \\

$\mC$ & Binary-code matrix of BCQ \\
$\valpha$  & Scale factors of BCQ \\
$z_B$ & Shifting factor of BCQ \\

$\vs$ & Element-wise scale factor of FlexRound \\
$s_r$ & Row-wise scale factor of FlexRound \\

$G$ & Grid search iterations \\
$T$ & Alternating update iterations \\
$p$ & Remapping period \\

\bottomrule
\end{tabular}
}
\end{table} 
\begin{table*}[t]
    \centering
    \caption{
    A summary of terminologies regarding different quantization methods.
    $\Rightarrow$ represents the application of the unification theorem (see Theorem~\ref{thm:unification}).
    \method's quantization parameters $\Theta_U$, transformation function $\trans_U$ and detransformation function $\detrans_U$ are determined according to the UQ method unified into the BCQ scheme. 
    }
    \label{tab:symbols_quant}
    \renewcommand{\arraystretch}{1.1}
    \resizebox{\linewidth}{!}{
    \begin{tabular}{ccccccc}
    \toprule
    \multirow{2}{*}{\textbf{Method}} & \multirow{2}{*}{\textbf{Quantizer}} & \textbf{Quantization} & \textbf{Transformation} & \textbf{Mapping} & \textbf{Detransformation} & \textbf{Reconstruction} \\
     & & \textbf{Parameters} & \textbf{Function} & \textbf{Function} & \textbf{Function} & \textbf{Function} \\ \midrule

    General form & $Q$ & $\Theta$ & $\trans(\vw; \Theta)$ & $\map(\vw; \Theta)$ & $\detrans(\vw; \Theta)$ & $\recon(\mC;\Theta)$ \\
    \midrule

    RTN & $Q_R$ & $\Theta_R =\{\Delta, z_{U}\} $ & $\trans_R(\vw; \Theta_R)$ & \multirow{2}{*}{$\umap(\vw,k)$} & \multirow{2}{*}{$\detrans_R(\vw; \Theta_R)$}  & \multirow{2}{*}{-} \\
    FlexRound & $Q_F$ & $\Theta_{F}=\{\Delta, z_U, \vs, s_r\}$ & $\ftrans(\vw; \Theta_F)$ & &  & \\

    \midrule

    \alternating & $Q_B$ & $\Theta_B=\{\valpha, z_{B}\}$ & $\btrans(\vw)$ & $\gmap(\vw; \Theta_B)$ & $\bdetrans(\vw)$ &
    $\recon(\mC;\Theta_B)$ \\

    \midrule

    \method & $Q_I$ &
    $\Theta_I = \Theta_U \cup \Theta_B \Rightarrow \Theta^*_B$ &
    $\trans_U(\vw; \Theta_U)$ &
    \multirow{2}{*}{$\lpgmap(\vw; \Theta_B)$} &
    $\detrans_U(\vw; \Theta_U)$ &
    \multirow{2}{*}{$\recon(\mC;\Theta^*_B)$} \\

    \methodf & $Q_{I_F}$ &
    $\Theta_{I_F}=\Theta_F \cup \Theta_B \Rightarrow \Theta^*_B$ &
    $\ftrans(\vw;\Theta_F)$ &
     &
    $\fdetrans(\vw;\Theta_F)$ &
     \\

    \bottomrule
    \end{tabular}
    }
    \vspace{3mm}
\end{table*}

\subsection{Symbols and Definitions}
\label{app:symbols}
 We summarize the frequently used symbols and their definitions in Tables~\ref{tab:symbols} and~\ref{tab:symbols_quant}.
%


\section{Inference Speed of BCQ Kernels} \label{app:lutgemm}
In this paper, we propose \method which unifies UQ~\cite{flexround,awq,omni,zsqsurvey,synq} and BCQ~\cite{xu,alphatuning} schemes, and the resulting models are represented in a BCQ scheme.
The effectiveness of \method depends on the inference speed of BCQ kernels~\cite{lut_gemm,shiftadd}.
Thus, we summarize the inference speed of LUT-GEMM~\cite{lut_gemm}, the state-of-the-art BCQ kernel, for completeness.

\setlength{\tabcolsep}{5pt}
\begin{table}[ht!]
    \centering
    \caption{Latency comparison of the first FFN layer on OPT-175B model with various precision and kernels on A100-80GB-GPU.}
    \label{tab:A_lg_1}
    \begin{tabular}{cccc}
    \toprule
    Kernel & Schemes & \# Bits & Latency (ms)  \\
    \midrule
    cuBLAS & - & 16 & 0.7256 \\
    GPTQ & UQ & 3 & 0.3599  \\
    AWQ & UQ & 4 & 0.3238  \\
    LUT-GEMM & UQ, BCQ & 4 & 0.2688  \\
    LUT-GEMM & UQ, BCQ & 3 & 0.2250  \\
    \bottomrule
    \end{tabular}
\end{table}

Table~\ref{tab:A_lg_1} summarizes the main results of Table 1 in LUT-GEMM~\cite{lut_gemm}.
In this table, latency represents the time for inferencing the first FFN layer in OPT-175B~\cite{opt} with various precision, and LUT-GEMM shows faster inference speeds than GPTQ~\cite{optq} and AWQ~\cite{awq} kernels, which support only a UQ scheme.

\begin{table}[t!]
    \centering
    \caption{Comparison of end-to-end latency per token for OPT-30B models on a A100-80GB-GPU.}
    \label{tab:A_lg_2}
    \begin{tabular}{ccccc}
    \toprule
    Model & Kernel-$k$-$g$ & Latency (ms) \\
    \midrule
    \multirow{5}{*}{\shortstack{OPT \\ -30B}} & cuBLAS-16-N/A & 40.5 \\
     & LUT-GEMM-4-32 & 18.5 \\
     & LUT-GEMM-4-64 & 17.8 \\
     & LUT-GEMM-3-32 & 16.7 \\
     & LUT-GEMM-3-64 & 15.7 \\
    \bottomrule
    \end{tabular}
\end{table}
\setlength{\tabcolsep}{6pt}

Table~\ref{tab:A_lg_2} summarizes the main results of Table 6 in LUT-GEMM~\cite{lut_gemm}, reporting the end-to-end latency per token for OPT family models.
$k$ and $g$ represent the bit width and group size, respectively.
As summarized in the table, LUT-GEMM provides end-to-end inference speedup with diverse bit-widths and group sizes.

In summary, LUT-GEMM provides faster inference speed for quantized models than existing UQ kernels and speeds up end-to-end inference.
Therefore, \methodf, which is supported by LUT-GEMM, is an essential method for quantizing LLMs.
Note that we need to convert quantized models using a UQ scheme into a BCQ scheme to use LUT-GEMM which provides a faster inference speed than existing UQ kernels as shown in Table~\ref{tab:A_lg_1}.
Therefore, quantized models in both schemes require the same memory and computational costs when they have the same bit width.


\section{Implementation Details} \label{app:impl}
We use PyTorch~\cite{pytorch} and Hugging Face~\cite{hf} libraries for implementation.
We use the pretrained weights of Llama-3~\cite{llama3} and Mistral~\cite{mistral} models from the Hugging Face library.
We discuss the implementation details of \methodf and competitors to reproduce the performance reported in this paper.
%

\subsection{Implementation Details of \methodf}\label{app:subsec:hyps}
\textbf{Hyperparameter settings.}
Our objective is to demonstrate that the outstanding performance of \methodf is achieved without expensive hyperparameter tuning although \methodf employs a block-wise output reconstruction process that requires many hyperparameters.
To this end, we fix all hyperparameters except for the grid size $G$ for unified initialization and utilize only two combinations of hyperparameters across all cases.
We outline these combinations in Table~\ref{tab:a_hyps}.

\setlength{\tabcolsep}{2pt}
\begin{table}[t!]
    \centering
    \caption{Hyperparameter settings of \methodf}\label{tab:a_hyps}
    \begin{tabular}{cc}
    \toprule
    \textbf{Hyperparameter} & \textbf{Setting} \\ \midrule
       Learning rate for $\Theta_F$ & 0.005 \\
       Learning rate for $\Theta_B$ & 0.0005 \\
       Grid search iterations ($G$) & 1, 30 \\
       Alternating update iterations ($T$) & 15 \\
       Remapping Period ($p$) & 2 \\
       Epochs & 20 \\
       Batch size & 1 \\
       Clipping strategy & Fixed-minimum
       \\ \bottomrule
    \end{tabular}
\end{table}
\setlength{\tabcolsep}{6pt}

\textbf{Gradient filtering.} 
We utilize a straight-through estimator (STE)~\cite{ste} to update \methodf's quantization parameters since its mapping function $\lpgmap$ is not differentiable.
We filter the gradient of the weights that have large mapping errors $|\bar{w}-\widetilde{w}|$ to stabilize the optimization process of \methodf, where $\bar{w}$ and $\widetilde{w}$ represent transformed and mapped weights, respectively.
STE hypothesizes that the gradient of a transformed weight $\bar{w}$ and mapped weight $\widetilde{w}$ have the same gradients.
However, if the difference between transformed and mapped weight is significant, the hypothesis does not hold which degrades the accuracy of quantized models.
Therefore, we set the hyperparameter $\tau$ as a gradient filtering threshold, and zero out the gradients of weights whose mapping error is larger than $\tau$.

We mimic the quantization process in the UQ scheme to determine $\tau$.
In UQ, weights within the clipping range $[w_{m,c}, w_{M,c}]$ are transformed to the values in the range $[0, 2^{k-1}]$, while weights outside the clipping range, such as $w_m$ and $w_M$, are transformed into values outside the range $[0, 2^{k-1}]$.
These out-of-range weights cause significant mapping errors.
In the transformed space, the interval between quantization levels is 1, and each level has a range of 0.5 on either side.
Thus, it is reasonable to allow a margin of 0.5 even for the smallest and largest quantization levels $0$ and $2^{k-1}$.

For \methodf, every quantization level has a potential risk of significant mapping errors, and the value of 0.5 is replaced with the smallest scale factor $min(\valpha)$ which is the smallest width in the binary tree constructed by BCQ's quantization levels (see Figure~\ref{fig:levels} (b)).
Therefore, we set the gradient filtering threshold $\tau$ as $min{(\valpha)}$.
%
\subsection{Implementation Details of Competitors}\label{app:subsec:comp_detail}
\textbf{RTN~\citep{rtn}.}
We search the clipping range with 100 iterations, which is larger than the number of iterations used in unified initialization.
We execute the same source code for clipping range search in \methodf.

\textbf{\alternating~\cite{xu}.}
We implement \alternating based on the original paper~\citep{xu}.
We use an alternating update with 15 iterations which is the same as the number of iterations used in unified initialization.
We use the same source code for the alternating update in \methodf.

\textbf{FlexRound~\citep{flexround}.}
We implement the FlexRound following the original paper~\citep{flexround}.

\textbf{OmniQuant~\citep{omni}.}
We refer to the official implementation\footnotemark[2] of OmniQuant and report the results following the best hyperparameter settings reported in the paper.



\section{Evaluation Protocol}
We report the average performance using random seeds 0, 1, and 2 except for Llama-3 70B which uses only 0 because of its long quantization time.
We sample 128 token sequences of length 2048 from C4~\cite{c4} and GSM8K~\cite{gsm8k} for general and task-specific knowledge evaluations, respectively.
We detail the evaluation protocols for general and task-specific knowledge evaluation as follows.

\textbf{General knowledge evaluation.}
We use MMLU~\cite{mmlu} and WikiText2~\cite{wiki} benchmarks for general knowledge evaluation.
MMLU consists of multiple-choice problems across 57 subjects.
We evaluate the 0-shot and 5-shot accuracies on MMLU to evaluate the amount of general knowledge in the quantized models.
0-shot and 5-shot refer to settings with 0 and 5 examples provided in the prompt, respectively.
We follow the evaluation protocol in the official code repository\footnote{\url{https://github.com/hendrycks/test}} for the MMLU benchmark.
%
WikiText2 consists of tokens within a set of verified articles from Wikipedia.
We report the perplexity of quantized models on WikiText2 benchmark to evaluate the language modeling capabilities of quantized models.
We follow the evaluation protocol used in OmniQuant~\cite{omni} according to its official implementation\footnote{\url{https://github.com/OpenGVLab/OmniQuant/tree/main}}.

\textbf{Task-specific knowledge evaluation.}
We use GSM8K~\cite{gsm8k} benchmark for task-specific knowledge evaluation.
GSM8K consists of high-quality grade school math problems that require multi-step reasoning,
and we use GSM8K to evaluate the amount of Mathematical knowledge in the quantized models.
We implement our evaluation code using the language model evaluation harness~\cite{eval-harness} package\footnote{\url{https://github.com/EleutherAI/lm-evaluation-harness/tree/main/lm_eval/tasks/gsm8k}}.

We summarize the properties of benchmarks in Table~\ref{tab:a_data}.

\begin{table}[h]
    \centering
    \caption{Properties of benchmarks}\label{tab:a_data}
    \begin{tabular}{cccc}
    \toprule
    \textbf{Benchmark} & \textbf{Subject} & \textbf{Instance} & \textbf{Metric} \\ \midrule
       MMLU & General & 14,042 & Accuracy \\
       WikiText2 & General & 141 & Perplexity \\
       GSM8K & Math & 1,319 & Accuracy
       \\ \bottomrule
    \end{tabular}
\end{table}


\section{Theoretical Details}
We explain the theoretical details on clipping strategy, general alternating update, local and periodic mapping, unification theorem, and extensibility of \method.

\subsection{Algorithm of Unified Initialization}\label{subsec:A_ui}
\begin{algorithm}[t!]
    \caption{Unified Initialization}
    \label{alg:ui}
    \begin{algorithmic}[1]
        \item[\textbf{Input:}] 
        Weights $\vw$, grid search iterations $G$, alternating update iterations $T$, and a bit-width $k$
        \item[\textbf{Output:}] 
        Initialized quantization parameters $\Theta_{I_F}$
        \STATE $w_m, w_M \gets min(\vw), max(\vw)$
        \STATE $z_B$, $\vs$, $s_r$, $e$ $\gets$ $(2^{k}-1)/2$, $\bm{1}_g$, 1, MAX\_NUM
        \vspace{-5mm}
        \FOR{$\gamma$ in $1/G, 2/G, ..., 1$}
        \STATE $\Delta'$, $z'_U$ $\gets$ adjust-clipping($w_m$,$w_M$,$\gamma$,$k$)
        \\ \hfill $\triangleright$ Appendix~\ref{subsec:A_clip}
        %
        \STATE $\bar{\vw}\gets\mathcal{T}_F(\vw;\Delta', z'_{U}, \vs, s_r)$
        \hfill  $\triangleright$ Equation~\ref{eq:ftrans}
        %
        \STATE $\valpha'$, $z'_B$ $\gets$general-alternating($\bar{\vw}$, $z_B$, $G$, $T$)
        \\ \hfill  $\triangleright$ Algorithm~\ref{alg:alt} in Appendix~\ref{sec:gau}
        \STATE $\Theta'_{I_F}$ $\gets$ $\{\Delta', z'_U, \vs, s_r, \valpha', z'_B\}$
        \STATE $\widehat{\vw}$ $\gets$ $Q^*_{I_F}(\vw;\Theta'_{I_F})$
         \hfill $\triangleright$ Equation~\ref{eq:qif}
        %
        \STATE $e'\gets || \vw-\widehat{\vw} ||_2^2$
        \hfill $\triangleright$ Quantization error
        \IF{$e'<e$}
            \STATE $\Delta$, $z_{U}$, $\valpha$, $z_B$, $e$
            $\gets$ $\Delta'$, $z'_{U}$, $\valpha'$, $z'_B$, $e'$
            %
        \ENDIF
        \ENDFOR
        \STATE $\Theta_{I_F}$ $\gets$ $\{\Delta, z_U, \vs, s_r, \valpha, z_B\}$
    \end{algorithmic}
\end{algorithm}
Algorithm~\ref{alg:ui} outlines the overall process of unified initialization which initializes the quantization parameters in the set $\Theta_{I_F}=\{ \Delta, z_U, \vs, s_r, \valpha, z_B \}$ for each weight group.
Following FlexRound~\cite{flexround}, we initialize $\vs$ and $s_r$ as $\bm{1}_g$ and 1, respectively.
We initialize the center $z_B$ of BCQ's quantization levels to $(2^{k}-1)/2$ since $\ftrans$ transforms weights $\vw$ into the range $[0, 2^{k}-1]$ (line 2).
Then, we conduct a grid search by adjusting the scale ratio $\gamma$
to find the optimal quantization parameters (lines 3-13).
We explore diverse values for the center of BCQ's quantization levels in the original weight space, by adjusting candidate scale factor $\Delta'$ and candidate zero-point $z'_U$ according to $\gamma$;
the weight transformed into $z_B$ changes as the adjusted $\Delta'$ and $z_U'$ modify the transformation function $\ftrans$.
There are Fixed-minimum, Fixed-maximum, and Balanced strategies for adjusting $\Delta'$ and $z'_U$.
The selection of a clipping strategy is a hyperparameter and we detail those strategies in Appendix~\ref{subsec:A_clip} (line 4).
Then, we transform $\vw$ into $\bar{\vw}$ using Equation~\ref{eq:ftrans} (line 5).
We find $\alpha'$ and $z'_B$ through a general alternating update in Algorithm~\ref{alg:alt} which is the improved version of alternating update~\cite{xu} to find $z_B$ (line 6).
After obtaining the candidate quantization parameters in $\Theta'_{I_F}$, we quantize $\vw$ into $\widehat{\vw}$ and evaluate the quantization error $e'$ (lines 7-9).
We update the quantization parameters using the candidate ones only if the evaluated error $e'$ is lower than the previous minimum error $e$ (lines 10-12).
After the grid search process over $G$ iterations, we obtain the initialized quantization parameters (line 14).

\subsection{Details of Clipping Strategies}\label{subsec:A_clip}
%

In unified initialization, we find the optimal quantization parameters by adjusting the clipping range [$w_{m,c}$, $w_{M,c}$] from altering a hyperparameter $\gamma$.
The length of the clipping range is adjusted as $\gamma(w_M-w_m)$ by setting scale factor $\Delta'=\gamma(w_M-w_m)/(2^k-1)$, and the clipping range is shifted according to the definition of the zero-point $z_U'$.
There are three strategies to define $z_U'$ as follows:
\begin{itemize}[leftmargin=3mm, itemsep=0mm, topsep=0mm]
    \item \textbf{Fixed-minimum.} $w_m$ is fixed at $w_{m,c}$ so that the $w_m$ is always included in the clipping range.
    In this case, we find $z_U'=-w_m/\Delta'$ by solving $\ftrans(w_m)=0$.
    \item \textbf{Fixed-maximum.} $w_M$ is fixed at $w_{M,c}$ so that the $w_M$ is always included in the clipping range.
    In this case, we find $z_U'=2^k-1-w_M/\Delta'$ by solving $\ftrans(w_M)=2^k-1$.
    \item \textbf{Balanced.} $w_{m,c}$ and $w_{M,c}$ are adjusted as $\gamma w_{m}$ and $\gamma w_{M}$, respectively, so that the minimum and maximum clipped values are adjusted in balance.
    In this case, we find $z_U'=-\gamma w_m/\Delta'$ by solving $\ftrans(\gamma w_m)=0$.
\end{itemize}
The center of BCQ's quantization levels explores more diverse values during grid search when we use Fixed-maximum or Fixed-minimum strategies than Balanced strategy; the clipping range is asymmetrically adjusted in Fixed-maximum and Fixed-minimum strategies.



\subsection{Algorithm of General Alternating Update}
\label{sec:gau}
\begin{algorithm}[t!]
    \caption{General alternating update}
    \label{alg:alt}
    \begin{algorithmic}[1]
        \item[\textbf{Input:}] 
        A weight group $\vw$ of size $g$. Numbers $G$ and $T$ of iterations for grid search and alternating update, respectively
        \item[\textbf{Output:}] 
        Updated $\valpha$ and $z_B$
        \STATE Greedily initialize $\valpha$ and $\mC$
        \\ \hfill $\triangleright$ Equation 4 in~\citet{xu}
        \FOR{$t$ in 1, 2, ..., $T$}
            \STATE $\valpha$ $\gets$ $(\mC^T\mC)^{-1}\mC^T(\vw-z_B \bm{1}_g)$
            \STATE $\mC$ $\gets$ $\arg\min_{\mC'}
            ||\vw - (\mC'\valpha + z_B \bm{1}_g)||_2^2$
            \STATE \textbf{if} {$G=\blue{1}$} \textbf{then}
            \\ \hfill $\triangleright$ Update $z_B$ only if not grid search
            \STATE $\quad z_B$ $\gets$ $\sum_{i=1}^{g} (w_i - \mC_{i,:}\valpha)/g$
            \STATE \textbf{end if}
        \ENDFOR
    \end{algorithmic}
\end{algorithm}
\vspace{2mm}


Algorithm~\ref{alg:alt} details the general alternating update process employed in line 6 of Algorithm~\ref{alg:ui} which is an improved version of alternating update in \citet{xu} to find proper $z_B$.
We initialize $\valpha$ and $\mC$ using a greedy initialization strategy (Equation (4) in \citet{xu}), where we determine $\valpha$ and $\mC$ for each bit-width one by one (line 1).
Then, we iteratively minimize the quantization error $e=||\vw - (\mC'\valpha + z_B\bm{1}_g)||_2^2$ for $T$ iterations to find the optimal quantization parameters (lines 2-8).
In each iteration, we optimize for $\valpha$, $\mC$, and $z_B$ sequentially while keeping the others fixed.
For $\valpha$ and $z_B$, we find the values where the derivatives $\frac{\partial e}{\partial \valpha}$ and $\frac{\partial e}{\partial z_B}$ equal to zero, respectively, where $\mC_{i,:}$ represents the $i$th row of $\mC$ (lines 3 and 6).
For $\mC$, we find the optimal binary codes for each weight by comparing errors obtained using all possible binary codes (line 4).
After completing these iterations, we obtain the optimized values for $\valpha$ and $z_B$.
We do not update $z_B$ when $G$ is greater than $1$ since diverse values for $z_B$ are explored during the iterative grid search process (line 5).


\subsection{Algorithm of Local and Periodic Mapping}
\label{alg:lpmapping}


\begin{algorithm}[t!]
    \caption{Local and periodic mapping ($\lpgmap$)}
    \label{alg:lpgmap}
    \begin{algorithmic}[1]
        \item[\textbf{Input:}] 
        A transformed weight group $\bar{\vw}$ of size $g$, iteration step $s$, a remapping period $p$, BCQ's quantization parameter $\Theta_B=\{\valpha, z_B\}$, and
        mapping indices $\vd$ and mapped weights $\widetilde{\vw}$ in the previous step
        \item[\textbf{Output:}] 
        Updated mapping indices $\vd$ and mapped weights $\widetilde{\vw}$
        \STATE \textbf{if} $s \equiv 0 \pmod p$ \textbf{then} \hfill \(\triangleright\) Periodic mapping

        \STATE
        $\quad\vq$ $\gets$ compute-quantization-levels($\valpha$, $z_B$)
        \STATE $\quad$\textbf{for} $i$ in $\{1, 2, ..., g\}$ \textbf{do}
            \STATE $\qquad l$, $r$ $\gets$ $max(d_i-1, 1)$, $min(d_i+1, 2^k)$
            \STATE $\qquad \mathcal{N}$ $\gets$ $\{l, d_i, r\}$ 
            \\ \hfill $\triangleright$ Neighbor quantization levels' indices 
            \STATE $\qquad d_i$ $\gets$ $\arg\min_{d'\in \mathcal{N}} |\bar{w}_i-q_{d'}|$
            \\ \hfill $\qquad \triangleright$ Local mapping
            \STATE $\qquad \widetilde{w}_i$ $\gets$ $q_{d_i}$
        \STATE $\quad$\textbf{end for}
        \STATE \textbf{end if}
    \end{algorithmic}
\end{algorithm}
    \vspace{2mm}


Algorithm~\ref{alg:lpgmap} describes the overall process of local and periodic mapping which swiftly updates the mapping between weights and quantization levels during optimization.
In the beginning, we initialize the mapping by calculating distances to all quantization levels.
Then, we begin local and periodic mapping using the resulting mapping $\vd$ and mapped weights $\widetilde{\vw}$ where $\vd$ contains the indices of quantization levels that the weights are mapped.
We periodically update the mappings between weights and quantization levels in every $p$ optimization step (line 1).
We compute quantization levels using $\valpha$ and $z_B$, and store them in $\vq\in\sR^{2^k}$ (line 2).
Then, we find the closest quantization level among its previously mapped level and its neighboring levels (lines 4-6).
We also update the mapped weight value to use them for the successive steps before the next update (line 7).
Local and periodic mapping is efficient, updating 1/$p$ times less often than non-periodic algorithms and using only 3 among $2^k$ quantization levels for computing the distance between weights.


\subsection{Proof of Theorem~\ref{thm:unification}}
\label{app:proof}
\begin{proof}
    By the definitions of $\mathcal{R}_B(\mC;\Theta_B)$ and $\mathcal{D}_R(\widetilde{\vw};\Theta_R)$, we reduce $\mathcal{D}_R(R_B(\mC;\Theta_B);\Theta_R)$ as follows.
    \begin{equation*}
    \begin{split}
        \mathcal{D}_R(\mathcal{R}_B&(\mC;\Theta_B);\Theta_R) \\
        &= \Delta((\mC \valpha + z_{B}\bm{1}_g) - z_{U}\bm{1}_g)\\
        &= \mC (\Delta\valpha) + \Delta (z_{B} - z_{U})\bm{1}_g.
    \end{split}
    \end{equation*}
    Thus, $\valpha^* = \Delta \valpha$, $z_{B}^* = \Delta (z_{B} - z_{U})$,
    and $R_B(\mC;\Theta^*_B) = \mC \valpha^* + z^*_{B}\bm{1}_g $
    where
    $\Theta_B^*=\{\valpha^*, z_{B}^*\}$.
\end{proof}

\subsection{Extensibility of \method}
\label{app:sec:extension}
In the main text, we employ \methodf which unifies FlexRound~\cite{flexround} and \alternating, the best-performing UQ and BCQ methods, respectively. 
We cover the unification between other quantization methods to explain the extensibility of \method.
The main idea of \method is to utilize the unified quantization process with quantization parameters $\Theta_I=\Theta_U \cup \Theta_B$ as follows:
\begin{equation*}
    Q_I(\vw;\Theta_I) = \udetrans(\gmap(\utrans(\vw;\Theta_U);\Theta_B);\Theta_U).
    \vspace{-1mm}
\end{equation*}
Considering the lack of accurate BCQ methods, we focus on replacing $\utrans$ and $\udetrans$ with method-specific ones to unify UQ methods into the BCQ scheme.

AWQ~\cite{awq} reduces output reconstruction error by scaling weights based on their importance before quantization.
This is accomplished by using $\trans_A$ and $\detrans_A$ defined as follows:
\begin{align}
\trans_A(\vw; \Theta_A) &= (\vw\odot \vs_A)/ \Delta_A  + z_A\bm{1}_g, \nonumber\\
\detrans_A(\widetilde{\vw}; \Theta_A) &= \Delta_A(\widetilde{\vw} - z_{O}\bm{1}_g)\oslash \vs_A, \nonumber
\end{align}
where $\vs_A\in\sR^g$ is a column-wise scale factor to reflect the importance of each column of the weight matrix where the weight group $\vw$ is located.
$\Delta_A\in\sR$ and $z_A\in\sR$ represent the new scale factor and zero-point after applying $\vs_A$, respectively.
$\Theta_A=\{\Delta_A, z_A, \vs_A\}$ is a set of quantization parameters for AWQ.
$\odot$ and $\oslash$ represent element-wise multiplication and division, respectively.
%

OmniQuant~\cite{omni} parameterizes the clipping range with parameters $\beta$ and $\gamma$ to find the best one through optimization.
This is accomplished by using $\otrans$ and $\odetrans$ as follows:
\begin{align}
\otrans(\vw; \Theta_O) &= \vw/ \Delta_O  + z_{O}\bm{1}_g, \nonumber\\
\odetrans(\widetilde{\vw}; \Theta_O) &= \Delta_O(\widetilde{\vw} - z_{O}\bm{1}_g), \nonumber
\end{align}
where $\Delta_O=(\gamma w_M - \beta w_m)/(2^k-1)$ and $z_{O}=-\lfloor \beta w_m /\Delta_O\rceil$ are the parameterized scale factor and zero-point, respectively.
$w_m$ and $w_M$ are the minimum and maximum weights in a weight group $\vw$.
$\Theta_O = \{\beta, \gamma\}$ is the set of quantization parameters of OmniQuant.

GPTQ~\cite{optq} is widely used in a UQ scheme, but it is a general error compensation technique which covers even pruning~\cite{sparsegpt}.
Thus, GPTQ is able to be unified into a BCQ scheme by applying GPTQ's error compensation strategy to BCQ methods.

In summary, it is able to unify diverse UQ methods into a BCQ scheme following \method's framework in theory.
Extending the coverage of unification to other quantization schemes beyond UQ and BCQ, and validating with experiments is one of our promising future works.

\section{Additional Analyses}
\label{app:analysis}
We analyze the memory usage, quantization time, hyperparameter sensitivity, and the effect of sample dataset size on \methodf.
We also validate that \methodf effectively leverages the optimization technique in FlexRound.
\begin{figure*}[t]
    \centering
    \includegraphics[width=0.95\textwidth]{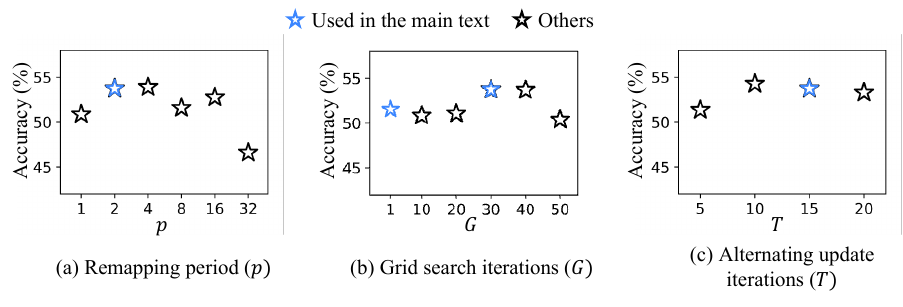}
    \vspace{-4mm}
    \caption{
    Change of average accuracy on 0-shot and 5-shot MMLU benchmarks with regard to the change of remapping period $p$, grid search iterations $G$, and alternating update iterations $T$.
    Blue stars represent the hyperparameters used in the main text and black stars represent the others.}
    \label{fig:A_sens}
\end{figure*}

\subsection{Memory Usage of \methodf}
We analyze the memory usage of quantized models using \methodf.
As summarized in Table~\ref{tab:schemes}, \methodf requires the same amount of memory at inference time as existing BCQ methods~\cite{xu,alphatuning} since we convert the inference process of the unified quantization process into a single BCQ's inference process using the unification theorem.
When we quantize a weight group $\vw\in\sR^{g}$ of size $g$ into $k$ bits,
we need to store a binary code matrix $\mC\in\{-1,+1\}^{g\times k}$, BCQ's scale factors $\valpha\in\sR^{k}$, and BCQ's shifting factor $z_B\in\sR$.
Each binary code requires 1 bit and each real number requires 16 bits for saving, thus we require $gk + 16(k+1)$ bits in total.
The memory overhead due to the quantization parameters is $16(k+1)/g$ bits per weight, and if we assume channel-wise quantization where $g$ is larger than 4000 for LLMs~\cite{mistral,llama3}, the memory overhead is negligible.

\subsection{Quantization Time of \methodf}
We compare the running time for quantizing Llama-3 8B into 3 bits and the accuracy of the quantized model, using \methodf and FlexRound to evaluate the efficiency of \methodf.
We compare them across different epoch settings since the quantization time varies depending on the number of epochs.
We do not include the time for unified initialization in \methodf since it is performed as a preprocessing before quantization;
the initialized quantization parameters are reused in different hyperparameter settings when the alternating update iteration $T$ and the grid search iteration $G$ are not changed.

\setlength{\tabcolsep}{3.5pt}
\begin{table}[h]
    \centering
    \caption{Comparison of average accuracies (\%) on 0-shot and 5-shot MMLU benchmarks, and quantization time (s) when quantizing Llama-3 8B into 3bit using \methodf and FlexRound, respectively.
    Bold and underlined texts represent the best and second-best results in each method, respectively.}
    \label{tab:A_time}
    \begin{tabular}{ccccc}
    \toprule
     &
      \multicolumn{2}{c}{\textbf{\methodf}} &
      \multicolumn{2}{c}{\textbf{FlexRound}} \\
    \multirow{-2}{*}{\textbf{Epochs}} &
      \textbf{Accuracy} &
      \textbf{Time} &
      \textbf{Accuracy} &
      \textbf{Time} \\
    \midrule
    5  & 47.97 & 5,089  & 46.71 & 4,239  \\
    10 & 51.04 & 10,022 & 47.97 & 8,350  \\
    15 & 52.08 & 15,905 & 51.08 & 12,405 \\
    20 & 53.46 & 20,403 & \textbf{51.19} & 16,868 \\
    25 & \textbf{54.23} & 25,884 & 48.19 & 21,970 \\
    30 & \underline{53.86} & 28,521 & \underline{51.09} & 25,139 \\
    \bottomrule
    \end{tabular}
\end{table}
\setlength{\tabcolsep}{6pt}

Table~\ref{tab:A_time} shows that \methodf requires about 20\% longer time for quantization than FlexRound, but it achieves higher accuracy than FlexRound across all epoch settings.
Especially, \methodf requires a shorter time of 15,905 seconds to outperform the highest accuracy of FlexRound which requires 16,868 seconds, demonstrating its efficiency.

The running time for the unified initialization process depends on $G$ and $T$.
When $G=30$ and $T=15$, as used in this experiment, it takes approximately 6,700 seconds and \methodf takes longer than FlexRound if we include this.
However, considering that the quantization time only needs to be performed once for deployment, the initialization results are reusable, and FlexRound cannot reach the high performance of \methodf even though we invest more time. 

\subsection{Sensitivity Analysis on $p$, $G$, and $T$}
We report the change in the performance of quantized models according to the variation of \methodf's hyperparameters $G$, $T$, and $p$ to promote reproducibility.
We report the average accuracy of 0-shot and 5-shot MMLU benchmarks of the 3-bit quantized Llama-3 8B models.

\textbf{Remapping period ($p$).}
Figure~\ref{fig:A_sens}(a) shows the change in the quantized models' accuracy with regard to the change of remapping period $p$.
We use $p\in\{1,2,4,8,16,32\}$ where a higher $p$ represents the sparse update of mapping between weights and quantization levels.
As a result, we find that the quantized model achieves the highest accuracy when $p\in\{2,4\}$, outperforming the case of $p=1$ where we update the mapping in every step.
This result indicates that periodic mapping not only improves the efficiency of \methodf but also improves the accuracy of the quantized models.
We recommend using $p=\{2, 4\}$ and we use only $p=2$ in this paper since they exhibit similar accuracies.

\textbf{Grid search iterations ($G$).}
Figure~\ref{fig:A_sens}(b) shows the change in the quantized models' accuracy with regard to the change of grid search iterations $G$.
We use $G\in\{{1}, 10, 20, 30, 40, 50\}$ where a higher $G$ represents the exhaustive search for unified initialization.
We also include {$G=1$} which indicates the case that we find the center of BCQ's quantization levels using the general alternating update in Algorithm~\ref{alg:alt}, unlike the other cases that use grid search.
As a result, we find that $G\in\{{1}, 30, 40\}$ shows the highest accuracy and we recommend using $G=\{{1}, 30\}$ to efficiently explore two different strategies for searching the BCQ's center.

\textbf{Alternating update iterations ($T$).}
Figure~\ref{fig:A_sens}(c) shows the change in the quantized models' accuracy with regard to the change of alternating update iterations $G$.
We use $T\in\{5, 10, 15, 20\}$ where a higher $T$ represents the more iterations for alternating update.
As a result, we find that \methodf achieves high accuracy when $T$ is equal to or higher than 10.
Thus, we recommend using $T=\{10, 15\}$ which shows the highest accuracy.

\subsection{Sensitivity on the Clipping Strategies}\label{subsec:A_clip_exp}
We compare the performance of three clipping strategies in unified initialization.
Table~\ref{tab:a_clipping} summarizes the comparison of the performance of Llama-3 8B models quantized using the three clipping strategies.

\begin{table}[t!]
    \centering
    \caption{The average accuracies on 0-shot and 5-shot MMLU benchmarks of 3-bit Llama-3 8B models with various clipping strategies.} \label{tab:a_clipping}
    \begin{tabular}{cc} \toprule
        \textbf{Strategy} & \textbf{Accuracy} \\ \midrule
        Fixed-minimum & 53.46 \\ 
        Fixed-maximum & 54.34 \\ 
        Balanced & 53.21 \\ \bottomrule
    \end{tabular}
\end{table}

Fixed-maximum and Fixed-minimum strategies outperform the Balanced strategy.
This is because the center of BCQ's quantization level explores more diverse values in Fixed minimum and Fixed maximum strategies than Balanced strategy, resulting in better quantization levels.
Therefore, we recommend evaluating Fixed minimum and Fixed maximum strategies and using the better one according to the experimental settings.

\subsection{Sensitivity on the Size of Sample Dataset}
To illustrate the effect of sample dataset size on the performance of \methodf, we quantize the Llama-3 8B model to 3 bits using sample datasets of varying sizes.
We evaluate the performance of quantized models on 0-shot and 5-shot MMLU benchmarks.
We use sample datasets ranging in size from 32 to 256, where 128 is the size used in the main text.
We summarize the result in Table~\ref{tab:a_sizes}.

\begin{table}[t!]
    \centering
    \caption{The average accuracies on 0-shot and 5-shot MMLU benchmarks of 3-bit Llama-3 8B models on various sizes of sample datasets.} \label{tab:a_sizes}
    \begin{tabular}{ccccc} \toprule
        \textbf{Sample Size} & \textbf{32} & \textbf{64} & \textbf{128} & \textbf{256} \\ \midrule
        Accuracy & 47.72 & 50.71 & 53.46 & 53.86  \\ \bottomrule
    \end{tabular}
\end{table}
%
Experimental results show that the accuracy of the quantized models increases as the sample size grows.
This is attributed to the fact a large number of data points promote the distillation of the general knowledge in the unquantized model to the quantized model.
Therefore, \methodf achieves higher performance than that reported in the main text by providing enough data points.

\subsection{Flexible Mappings in \methodf} \label{subsec:app:flex}
\methodf unifies FlexRound~\cite{flexround} and \alternating.
We perform an in-depth analysis to verify that \methodf effectively leverages the main optimization technique of FlexRound.
The main idea of FlexRound is ``flexible mapping" which introduces additional scale factors $\vs$ and $s_r$ to make weights explore diverse quantization levels, and selects the best one.
We compare the amount of flexible mapping occurred in \methodf and FlexRound for validation.
We use a 3-bit quantized Llama-3 8B model for analysis.

\textbf{Case study.} We analyze the amount of flexible mappings in a weight group of quantized models and visualize the results.
Figure~\ref{fig:flexmap} represents the index difference of mapped quantization levels for each weight before and after applying scale factors $\vs$ and $s_r$ when quantizing Llama-3 8B models into 3-bit using \methodf and FlexRound.
An index difference equal to or larger than 1 represents that the weight experiences the flexible mapping through $\vs$ and $s_r$.
Note that plenty of flexible mappings occur in the quantization process of \methodf similarly in FlexRound.

\begin{figure}[t!]
    \centering
    \includegraphics[width=0.49\textwidth]{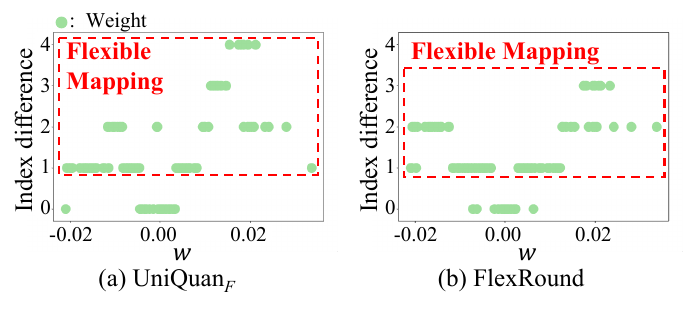}
    \vspace{-4mm}
    \renewcommand{\arraystretch}{0.97}
    \caption{
    Index difference of mapped quantization levels per weight before and after applying $\vs$ and $s_r$.
    Flexible mappings occur sufficiently in \methodf compared to FlexRound.
    }
    \label{fig:flexmap}
\end{figure}

\setlength{\tabcolsep}{3.5pt}
\begin{table}[t!] \centering
\caption{{Proportion of weights exhibiting flexible mapping across the entire model. Each column indicates the index difference of the mapped quantization level resulting from flexible mapping.}}\label{tab:a_fmM}
\begin{tabular}{@{}ccccc@{}}
\toprule
\multirow{2}[3]{*}{\textbf{Method}}& \multicolumn{4}{c}{\textbf{Index Difference}}\\
\cmidrule(lr){2-5}
 & \textbf{0} & \textbf{1} & \textbf{2} & \textbf{$>$2} \\ \midrule
FlexRound       & 95.4132    & 4.5862     & 5.00e-04   & 1.30e-05    \\
\methodf         & 96.0395    & 3.8619     & 9.85e-02   & 7.62e-05    \\ \bottomrule
\end{tabular}
\end{table}
\setlength{\tabcolsep}{6pt}

\textbf{Global pattern.}
We analyze the proportion of flexible mapping applied across all model weights to examine that flexible mappings occur globally.
Table~\ref{tab:a_fmM} illustrates the proportion of the flexibly mapped weights in the entire model.
The column names in the tables indicate the amount of changes in indices of weights' mapped quantization levels resulting from flexible mapping.
Across the entire model, \methodf demonstrates a similar level of flexible mapping as FlexRound.
Therefore, \methodf effectively utilizes flexible mapping across the entire model as we intended.

\section{Use of AI Assistant}
We use ChatGPT\footnote{\url{https://chatgpt.com/}} and Gemini\footnote{\url{https://gemini.google.com/}} only for grammar checking and sentence re-wording purposes.
We do not use them for research purposes such as developing our main ideas or analyzing our experimental results.

\section{Potential Risks}
In this paper, we propose \methodf, a quantization method for large language models (LLMs), and there is a potential risk of losing the models' knowledge during quantization.
In experiments, we rigorously validate the amount of knowledge loss in both general and task-specific aspects to verify whether the risk occurs.
As a result, we demonstrate that \methodf results in significantly less knowledge loss than other methods, confirming its low risk.



\end{document}